\documentclass[sigconf]{acmart}

\usepackage{caption}
\usepackage{subcaption}
\usepackage[linesnumbered,ruled]{algorithm2e}
\usepackage{booktabs} % For formal tables
\usepackage{paralist}
\usepackage{bm}
\usepackage{mathtools}
\usepackage{xspace}
\usepackage{enumitem}
\usepackage{comment}

\let\en=\ensuremath

\DeclarePairedDelimiter{\norm}{\lVert}{\rVert}

\DeclareMathOperator{\E}{\mathbb{E}}          % expectation
\newcommand{\Lap}[1]{\en{\mathsf{Lap}(#1)}}   % Laplace distribution

\DeclareMathOperator*{\argmax}{arg\,max\,}
\DeclareMathOperator{\Range}{range}

\renewcommand{\vec}[1]{\en{\bm{\mathrm{#1}}}}
\newcommand{\mat}[1]{\en{{\bm{\mathrm{#1}}}}}
\newcommand{\grad}[0]{\en{\nabla}}

\newcommand{\R}[0]{\mathbb{R}}

\SetFuncSty{textsc}

\setcopyright{none}

\acmDOI{}

\acmISBN{}

\acmConference{}{}{}
\acmYear{}
\begin{document}
\title{Concentrated Differentially Private Gradient Descent with Adaptive per-Iteration Privacy
  Budget}
% \titlenote{Produces the permission block, and
% \subtitle{Extended Abstract}
% \subtitlenote{The full version of the author's guide is available as
%   \texttt{acmart.pdf} document}

\author{Jaewoo Lee}
\affiliation{%
  \institution{University of Georgia}
  % \streetaddress{Dept. of Computer Science}
  \city{Athens}
  \state{GA}
  \postcode{30602}
}
\email{jwlee@cs.uga.edu}

\author{Daniel Kifer}
% \authornote{The secretary disavows any knowledge of this author's actions.}
\affiliation{%
  \institution{Penn State University}
  % \streetaddress{P.O. Box 1212}
  \city{University Park}
  \state{PA}
  \postcode{16802}
}
\email{dkifer@cse.psu.edu}

\renewcommand{\shortauthors}{J. Lee and D. Kifer}

\graphicspath{{./figures/}}

\begin{abstract}
  Iterative algorithms, like gradient descent, are common tools for solving a variety of problems, such as model fitting. For this reason, there is interest in creating differentially private versions of them. However, their conversion to differentially private algorithms is often naive. For instance, a fixed number of iterations are chosen, the privacy budget is split evenly among them, and at each iteration, parameters are updated with a noisy gradient.

In this paper, we show that gradient-based algorithms can be improved by a more careful allocation of privacy budget per iteration. Intuitively, at the beginning of the optimization, gradients are expected to be large, so that they do not need to be measured as accurately. However, as the parameters approach their optimal values, the gradients decrease and hence need to be measured more accurately. We add a basic line-search capability that helps the algorithm decide when more accurate gradient measurements are necessary.

Our gradient descent algorithm works with the recently introduced zCDP version of differential privacy. It outperforms prior algorithms for model fitting and is competitive with the state-of-the-art for $(\epsilon,\delta)$-differential privacy, a strictly weaker definition than zCDP.

%Iterative optimization algorithms for differential privacy derive a
%sequence of noisy statistics from a sensitive dataset, in which the
%utility of each statistic is determined by the privacy budget assigned
%to each iteration. The quality of final solution is totally dependent
%on the utility of each noisy statistic. While there exists certain
%utility requirements an algorithm need to satisfy for convergence,
%existing algorithms largely ignore them and use predetermined sequence
%of privacy budgets, providing no guarantee on convergence.
%
%In this work, we propose a gradient-based algorithm that takes the
%requirements into account and adaptively chooses both per-iteration
%privacy budget and step size at runtime, based on the utility and
%previous iterate. The main idea of the proposed
%algorithm is to explicitly check the utility of statistic (using a
%small portion of privacy budget) and to iteratively reduce the noise
%in the statistic when it has a poor utility.
%We evaluate the performance of our algorithm by performing regularized
%logistic regression and support vector machine tasks on a variety of
%real datasets.

%%% Local Variables:
%%% mode: latex
%%% TeX-master: "main"
%%% End:

\end{abstract}




\maketitle

\section{Introduction}
Iterative optimization algorithms are designed to find a parameter vector $\vec{w}^*\in\R^p$ that minimizes an objective function $f$. They start with an initial guess
 $\vec{w}_0$ and generate a 
sequence of iterates $\{\vec{w}_t\}_{t\geq 0}$ such that $\vec{w}_t$
tends to $\vec{w}^*$ as $t \to \infty$. At iteration $t$, information about the objective function $f(\vec{w}_t)$, such as the gradient $\grad f(\vec{w}_t)$, is computed and used to obtain the next iterate $\vec{w}_{t+1}$. In the case of gradient (or stochastic gradient) descent, updates have a form like:
\begin{equation}
  \vec{w}_{t+1} = \vec{w}_t - \alpha_t (\grad f(\vec{w}_t))\,,
  \label{eq:nonprivate_update}
\end{equation}
where $\alpha_t$ is a carefully chosen step size that often depends on the data (for example, through a line search \cite{Nocedal2006NO}) or based on previous gradients.

%certain
%information about objective function $f$ (e.g., objective values
%$f(\vec{w}_t)$, gradients of objective function $\grad f(\vec{w}_t)$),
%denoted by $S(\vec{w}_t)$, is derived from a dataset and used to
%generate the next iterate from the previous one.  
%Furthermore, to ensure convergence of the iterative process, many
%algorithms in this category insist certain decrease of the objective
%function over iterations.
When designing a such algorithm under various versions of differential privacy, the update steps typically have the following form \cite{song13, Abadi2016deep}:
\begin{equation}
  \vec{w}_{t+1} = \vec{w}_t - \alpha_t (\grad f(\vec{w}_t) + Y_t)\,,
  \label{eq:gd_update}
\end{equation}
where $Y_t$ is an appropriately scaled noise variable (e.g., Laplace or Gaussian) for iteration $t$, and the gradient may be computed on some or all of the data. It is important to note that the noisy gradients $\grad f(\vec{w}_t) + Y_t$ might not be descent directions even when computed on the entire dataset. 

In prior work (e.g., \cite{Zhang2013privgene,Bassily2014PERM,Talwar2015nearly,Wang2015free}), the total number of iterations $T$ is fixed \emph{a priori}, and the desired privacy cost, say $\epsilon$, is split across the iterations: $\epsilon=\epsilon_1+\dots+\epsilon_T$. For any iteration $t$, the variance of $Y_t$ is a function of $1/\epsilon_t$ and depends on which version of differential privacy is being used (e.g., pure differential privacy \cite{Dwork2006calibrating}, approximate differential privacy \cite{Dwork2006our}, or zero-mean concentrated differential privacy \cite{Bun2016zCDP}). Furthermore, in prior work, the privacy budget is evenly split across iterations, so $\epsilon_1=\dots=\epsilon_T=\epsilon/T$.

There are two drawbacks to this approach. First, accuracy heavily depends on the pre-specified number of iterations $T$ --- if $T$ is too small, the algorithm will stop well short of the optimum; if $T$ is too large, the privacy budget $\epsilon_t$ for each iteration is small, so that large amounts of noise must be added to each gradient, thus swamping the signal provided by the gradient. Second, at the beginning of the optimization, gradients are expected to be large, so that an algorithm can find good parameter updates even when the gradient is not measured accurately. However, as the current parameters $\vec{w}_t$ approach the optimal values, the gradients start to decrease and need to be measured more accurately in order for the optimization to continue making progress (e.g., continue to minimize or approximately minimize $f$). This means that an adaptive privacy budget allocation is preferable to a fixed allocation (as long as the total privacy cost is the same).

In this paper, we propose an adaptive gradient descent strategy for zero-mean Concentrated Differential Privacy \cite{Bun2016zCDP} (zCDP) where each iteration has a different share $\epsilon_t$ of the overall privacy budget $\epsilon$. It uses a smaller share of the privacy budget (more noise) for gradients with large norm and a larger share (less noise) for gradients with small norm. Thus, if there are many steps with large gradients, the algorithm will be able to run for more iterations, while if there are many steps with small gradients, it will run for fewer iterations but make sure that each noisy gradient is accurate enough to help decrease the objective function (instead of performing a completely random walk over the parameter space). To the best of our knowledge, our work is the first to adaptively choose $\epsilon_t$ depending on the previous iterate and the utility of noisy statistic for the current iteration.

One of the challenges is to figure out whether the amount of noise added to a gradient is too much to be useful. This is far from trivial as the noisy gradient can be a descent direction even when the norm of the noise is much larger than the norm of the true gradient. For example, consider the following run of the noisy gradient descent  algorithm to train logistic regression on the UCI Adult dataset \cite{Lichman2013UCI} with Gaussian noise vectors added to the gradient, as shown in Table \ref{tab:graddesc}.
\begin{table}[t!]
\begin{tabular}{|l|l|l|l|l|}\hline
$t$ & $||\grad f(\vec{w}_t)||_2$ & $||\grad f(\vec{w}_t)+Y_t||_2$ & $\sqrt{E[||Y_t||_2^2]}$ & $f(\vec{w}_t)$\\\hline
 0 &    0.72558 &    0.91250 &    0.50119 &    0.69315 \\
 1 &    0.20550 &    0.52437 &    0.50119 &    0.53616 \\
 2 &    0.15891 &    0.54590 &    0.50119 &    0.46428 \\
 3 &    0.11864 &    0.49258 &    0.50119 &    0.43678 \\
 4 &    0.09715 &    0.50745 &    0.50119 &    0.41852 \\
 5 &    0.14050 &    0.52271 &    0.50119 &    0.41122 \\
 6 &    0.12380 &    0.48218 &    0.50119 &    0.38903 \\
 7 &    0.06237 &    0.53640 &    0.50119 &    0.38175 \\
 8 &    0.05717 &    0.47605 &    0.50119 &    0.37865 \\
 9 &    0.05625 &    0.55129 &    0.50119 &    0.37814 \\
10 &    0.05241 &    0.51636 &    0.50119 &    0.37542 \\\hline
\end{tabular}
\caption{Objective function value and true gradient vs. noisy gradient magnitude. \label{tab:graddesc}}
\end{table}
Note that Gaussian noise is one of the distributions that can achieve zero-Mean Concentrated Differential Privacy \cite{Bun2016zCDP}.
We see that the magnitude of the true gradient decreases from approximately $0.7$ to $0.05$ while the norm of the noisy gradient starts at $0.91$ and only decreases to approximately $0.516$ -- an order of magnitude larger than the corresponding true gradient. Yet, all this time the objective function keeps decreasing, which means that the noisy gradient was still a descent direction despite the noise.

Our solution is to use part of the privacy budget allocated to step $t$ to compute the noisy gradient $\tilde{S}_t=\grad f(\vec{w}_t) + Y_t$. We use the remaining part of the privacy budget allocated to step $t$ to select the best step size. That is, we start with  a predefined set of step sizes $\Phi$ (which includes a step size of 0). Then, we use the differentially private noisy min algorithm \cite{Dwork2014DPbook} to approximately find the $\alpha\in \Phi$ for which  $f(\vec{w}_t-\alpha \tilde{S}_t)$ is smallest (i.e. we find which step size causes the biggest decrease on the objective function).  If the selected step size $\alpha$ is not $0$, then we set $\vec{w}_{t+1}=\vec{w}_t - \alpha \tilde{S}_t$; thus our algorithm supports variable step sizes, which can help gradient descent algorithms converge faster. On the other hand, if the selected step size is $0$, it is likely that the noise was so large that the noisy gradient is not a descent direction and it triggers an increase in share of the privacy budget that is assigned to subsequent steps.

This brings up the second problem. If the chosen step size $\alpha$ is
$0$, it means two things: we should increase our current privacy
budget share from $\epsilon_t$ to some larger value
$\epsilon_{t+1}$. It also means we should not use the current noisy
gradient for a parameter update. However, the noisy gradient still
contains some information about the gradient. Thus, instead of
measuring the gradient again using a privacy budget of
$\epsilon_{t+1}$ and discarding our previous estimate, we
measure it again with a smaller budget $\epsilon_{t+1}-\epsilon_t$ and
merge the result with our previous noisy gradient.

Our contributions are summarized as follows:
\begin{itemize}[leftmargin=*]
\item We propose a gradient descent algorithm for a variation of differential privacy, called zCDP \cite{Bun2016zCDP}, that is weaker than $\epsilon$-differential privacy, but is stronger than  $(\epsilon,\delta)$-differential privacy.
\item To the best of our knowledge, this is the first private gradient-based algorithm in which the privacy budget and step size for each iteration is dynamically determined at runtime based  on the quality of the noisy statistics (e.g., gradient) obtained for the current
  iteration.
%\item To ensure that each iteration improves the quality of
%  approximation, we combine two existing algorithms, NoisyMax and
%  NoiseDown, under the zCDP framework of~\cite{Bun2016zCDP}.
%\item We present an extension of NoiseDown \cite{Xiao2011iReduct} to the Gaussian case that makes it compatible with zCDP and $(\epsilon,\delta)$-differential privacy. This is a tool of independent interest, as it can be used in other applications, such as releasing marginals with relative error guarantees \cite{Xiao2011iReduct} under various relaxations of differential privacy.
\item We perform extensive experiments on real datasets against other recently proposed empirical risk minimization algorithms. We empirically
  show the effectiveness of the proposed algorithm for a wide range of
  privacy levels. 
\end{itemize}

The rest of this paper is organized as follows. In
Section~\ref{sec:related_work}, we review related work. In
Section~\ref{sec:background}, we provide background on differential
privacy. Section~\ref{sec:gradavg} introduces our gradient averaging
technique. We present the approach for the dynamic adaptation of privacy
budget in Section~\ref{sec:algorithm}. Section~\ref{sec:experiments}
contains the experimental results on real datasets.

%%% Local Variables:
%%% mode: latex
%%% TeX-master: "main"
%%% End:

\section{Related Work}
\label{sec:related_work}

A typical strategy in statistical learning is the empirical risk
minimization (ERM), in which a model's averaged error on a dataset is
minimized.
There have been several
efforts~\cite
{Williams2010probabilistic,Chaudhuri2011Objpert,Rubinstein2012learning,Kifer2012erm,Jain2012online,Zhang2013privgene,Bassily2014PERM,Wang2015free,Talwar2015nearly,Wang2017dpsvrg}
to develop privacy-preserving algorithms for convex ERM
problems using variations of differential privacy. A number of approaches have been proposed in the literature. 
The simplest approach is to perturb the output of a non-private
algorithm with random noise drawn from some probability
distribution. This is called \emph{output
  perturbation}~\cite{Dwork2006calibrating,Chaudhuri2011Objpert,Zhang2017RR}. In
general, the resulting noisy outputs of learning algorithms are often
inaccurate because the noise is calibrated to the worst case
analysis. Recently, Zhang et al.~\cite{Zhang2017RR} used algorithmic stability arguments to bound the $L_2$  sensitivity of full batch gradient descent
algorithm to determine the amount of noise that must be added to outputs that partially optimizes the objective function. Although they achieve theoretical near optimality, this algorithm has not been empirically shown to be superior to methods such as \cite{Chaudhuri2011Objpert}.
%, and showed that their
%output perturbation method achieves nearly optimal utility.

One approach that has shown to be very effective is the \emph{objective
  perturbation} method due to Chaudhuri et
al.~\cite{Chaudhuri2011Objpert}. In objective
perturbation, the ERM objective function is perturbed by
adding a linear noise term to its objective function, and then the
problem is solved using a non-private optimization solver.
Kifer et al.~\cite{Kifer2012erm} improved the utility of the objective
perturbation method at the cost of using approximate instead of pure differential privacy.
While this
approach is very effective, its privacy guarantee is based on the
premise that the problem is solved \emph{exactly}. This is, however,
rarely the case in practice; most of time optimization problems are
solved approximately.

Another approach that has gained popularity is the iterative gradient
perturbation method~\cite{Williams2010probabilistic,Bassily2014PERM} and their
variants~\cite{Wang2017dpsvrg,Wang2015free,Zhang2017RR}. Bassily et
al.~\cite{Bassily2014PERM} proposed an $(\epsilon,
\delta)$-differentially private version of stochastic gradient descent
(SGD) algorithm. At each iteration, their algorithm perturbs the
gradient with Gaussian noise and applies the advanced
composition~\cite{Dwork2010boosting} together with privacy
amplification result~\cite{Beimel2014bounds} to get an upper bound on the total 
privacy loss. Further, they also have shown that their lower bounds on
expected the excess risk is optimal, ignoring multiplicative log factor for
both lipschitz convex and strongly convex functions. Later, Talwar et
al.~\cite{Talwar2015nearly} improved those lower bounds on the utility
for LASSO problem. In~\cite{Wang2017dpsvrg}, gradient perturbation
method has been combined with the stochastic variance reduced gradient (SVRG)
algorithm, and the resulting algorithm has shown to be near-optimal
with less gradient complexity.

Zhang et al.~\cite{Zhang2013privgene} presented a genetic algorithm for
differentially private model fitting, called PrivGene, which has a
different flavor from other gradient-based methods. Given the fixed
number of total iterations, at each iteration, PrivGene 
iteratively generates a set of candidates by emulating natural
evolutions and chooses the one that best fits the model using the
exponential mechanism~\cite{Dwork2014DPbook}.

All of the iterative algorithms discussed above use predetermined
privacy budget sequence.

%%% Local Variables:
%%% mode: latex
%%% TeX-master: "main"
%%% End:

\section{Background}
\label{sec:background}
In this section, we provide background on differential privacy and
introduce important theorems.

\subsection{Differential Privacy}
Let $D = \{d_1, d_2, \ldots, d_n\}$ be a set of $n$
 observations, each drawn from some domain $\mathcal{D}$.
A database $D' \in \mathcal{D}^n$ is called neighboring to $D$ if
$|(D\setminus D') \cup (D' \setminus D)| = 1$.
In other words, $D'$ is obtained by adding or removing one
observation from $D$. To denote this relationship, we write
$D{\sim}D'$. 
The formal definition of differential privacy (DP) is given in
Definition~\ref{def:dp}. 
\begin{definition}[($\epsilon,\delta$)-DP~\cite{Dwork2006calibrating,Dwork2006our}]
A randomized mechanism $\mathcal{M}$ satisfies ($\epsilon,
\delta$)-differential privacy if for every event $S \subseteq
\Range(\mathcal{M})$ and for all $D{\sim}D' \in \mathcal{D}^n$, 
\[
  \Pr[\mathcal{M}(D) \in S] \leq \exp(\epsilon)\Pr[\mathcal{M}(D')\in
  S] + \delta\,. 
\]
\label{def:dp}
\end{definition}
When $\delta=0$, $\mathcal{M}$ achieves \emph{pure} differential privacy
which provides stronger privacy protection than \emph{approximate}
differential privacy in which $\delta > 0$.

To satisfy
($\epsilon$, $\delta$)-DP (for $\delta>0$), we can use the Gaussian mechanism, which adds
Gaussian noise calibrated to the $L_2$ sensitivity of the query
function.
\begin{definition}[$L_1$ and $L_2$ sensitivity]
Let $q:\mathcal{D}^n \to \R^d$ be a query function. The $L_1$ (resp. $L_2$) sensitivity of
$q$, denoted by $\Delta_1(q)$ (resp., $\Delta_2(q)$) is defined as
\[
\Delta_1(q) = \max_{D{\sim}D'}\norm{q(D) - q(D')}_1\qquad \Delta_2(q) = \max_{D{\sim}D'} \norm{q(D) - q(D')}_2\,.
\]
\end{definition}
The $L_1$ and $L_2$ sensitivities represent the maximum change in the output
value of $q$ (over all possible neighboring databases in
$\mathcal{D}^n$) when one individual's data is changed.
\begin{theorem}[Gaussian mechanism~\cite{Dwork2014DPbook}]
Let $\epsilon \in (0, 1)$ be arbitrary and $q$ be a query function
with $L_2$ sensitivity of $\Delta_2(q)$. The Gaussian Mechanism, which returns $q(D)+N(0,\sigma^2)$, with
\begin{equation}
  \sigma \geq \frac{\Delta_2(q)}{\epsilon}\sqrt{2\ln(1.25/\delta)}
  \label{eq:gm}
\end{equation}
is ($\epsilon$, $\delta$)-differentially private.
\label{thm:gaussian_mech}
\end{theorem}
An important property of differential privacy is that its privacy
guarantee degrades gracefully under the composition. The most basic 
composition result shows that the privacy loss grows \emph{linearly}
under $k$-fold composition~\cite{Dwork2014DPbook}.
This means that, if we sequentially apply an $(\epsilon, \delta)$-DP algorithm
$k$ times on the same data, the resulting process is $(k\epsilon,
k\delta)$-differentially private. Dwork et
al.~\cite{Dwork2010boosting} introduced an advanced composition, where
the loss increases sublinearly (i.e., at the rate of $\sqrt{k}$).
\begin{theorem}[Advanced composition~\cite{Dwork2010boosting}]
For all $\epsilon$, $\delta$, $\delta'\geq 0$, the class of $(\epsilon,
\delta)$-differentially private mechanisms satisfies $(\epsilon',
k\delta + \delta')$-differential privacy under $k$-fold adaptive
composition for $\epsilon'=\sqrt{2k\ln(1/\delta')}\epsilon +
k\epsilon(e^\epsilon-1)$. 
\label{thm:advcomp}
\end{theorem}

\subsection{Concentrated Differential Privacy}
Bun and Steinke~\cite{Bun2016zCDP} recently introduced a relaxed 
version of differential privacy, called zero-concentrated differential
privacy (zCDP). To define $\rho$-zCDP, we first introduce the privacy
loss random variable. For an output $o \in \Range(\mathcal{M})$, the
privacy loss random variable $Z$ of the mechanism $\mathcal{M}$ is defined as
\[
  Z = \log\frac{\Pr[\mathcal{M}(D) = o]}{\Pr[\mathcal{M}(D')=o]}\,.
\]
$\rho$-zCDP imposes a bound on the \emph{moment generating function} of the
privacy loss $Z$ and requires it to be concentrated around
zero. Formally, it needs to satisfy
\[
  e^{D_{\alpha}(\mathcal{M}(D)||\mathcal{M}(D'))}=\E\left[e^{(\alpha-1)Z}\right] \leq
  e^{(\alpha-1)\alpha\rho}\,,\; \forall \alpha \in (1, \infty)\,, 
\]
where $D_{\alpha}(\mathcal{M}(D)||\mathcal{M}(D'))$ is the
$\alpha$-R\'enyi divergence.
In this paper, we use the following zCDP composition results.
\begin{lemma}[\cite{Bun2016zCDP}]
  Suppose two mechanisms satisfy $\rho_1$-zCDP and $\rho_2$-zCDP, then
  their composition satisfies ($\rho_1+\rho_2$)-zCDP.
  \label{lem:comp}
\end{lemma}
\begin{lemma}[\cite{Bun2016zCDP}]
  The Gaussian mechanism, which returns $q(D)+N(0,\sigma^2)$  satisfies $\Delta_2(q)^2/(2\sigma^2)$-zCDP.
  \label{lem:gauss2zcdp}
\end{lemma}
\begin{lemma}[\cite{Bun2016zCDP}]
  If $\mathcal{M}$ satisfies $\epsilon$-differential privacy, them
  $\mathcal{M}$ satisfies $(\frac{1}{2}\epsilon^2)$-zCDP.
  \label{lem:eps2zcdp}
\end{lemma}
\begin{lemma}[\cite{Bun2016zCDP}]
  If $\mathcal{M}$ is a mechanism that provides $\rho$-zCDP, then
  $\mathcal{M}$ is 
  ($\rho+2\sqrt{\rho\log(1/\delta)}, \delta$)-DP for any $\delta>0$.
  \label{lem:zcdp2dp}
\end{lemma}

\subsection{NoisyMax}
Let $\Psi = \{\vec{w}_1, \ldots, \vec{w}_s\}$ be a set of points in $\R^p$ and
$f:\R^p \to \R$ be a function that implicitly depends on a database
$D$. Suppose we want to choose a point
$\vec{w}_i\in \Psi$ with maximum $f(\vec{w}_i;D)$. There exists an
$(\epsilon, 0)$-DP algorithm, called 
NoisyMax~\cite{Dwork2014DPbook}. It adds independent 
noise drawn from $\Lap{\Delta_1(f)/\epsilon}$ to each $f(\vec{w}_i)$, for
$i\in [s]$, and
returns the index $i$ of the largest value, i.e.,
\[
  i = \argmax_{j \in [s]} \,\{f(\vec{w}_{j}) + \Lap{\Delta_f / \epsilon}\}\,,
\]
where $\Lap{\lambda}$ denotes
a Laplace distribution with mean 0 and scale parameter $\lambda$, and
the notation $[s]$ is used to denote the set $\{1, 2, \ldots, s\}$.
Note that, when $f$ is \emph{monotonic} in $D$
(i.e., adding a tuple to $D$ cannot decrease the value of $f$), noise
can be drawn from the exponential distribution with parameter
$\epsilon/\Delta_1(f)$, which yields better utility. The NoisyMin algorithm is obtained by applying NoisyMax to $-f$.

NoisyMax was originally intended to work with pure $\epsilon$-differential privacy. To get it to work with $\rho$-zCDP, we use the conversion result in Lemma \ref{lem:eps2zcdp}: an $\epsilon$-differentially private algorithm satisfies $\frac{\epsilon^2}{2}-zCDP$. Therefore, when using zCDP, if we wish to allocate $\rho^\prime$ of our zCDP privacy budget to NoisyMax, we call NoisyMax with $\epsilon=\sqrt{2\rho^\prime}$.

\begin{algorithm}[tp]
  \DontPrintSemicolon
  \KwIn{$\Omega$: a set of candidates, $\Delta_1(f)$: sensitivity of $f$,
    $\epsilon$: privacy budget for pure differential privacy}
  $\widetilde{\Omega} = \{\tilde{v}_i = v + \Lap{\Delta_1(f)/\epsilon}: v \in
  \Omega, i \in [|\Omega|]\}$\;
  \Return $\argmax_{j \in [|\Omega|]} \,\tilde{v}_j$\;
  \caption{\scshape NoisyMax($\Omega, \Delta_1(f), \epsilon$)}  
\end{algorithm}

\section{Gradient Averaging for zCDP}
\label{sec:gradavg}
One of the components of our algorithm is recycling estimates of gradients that weren't useful for updating parameters. In this section, we explain how this is done.
Suppose at iteration $t$, we are allowed to use $\rho^\prime_t$ of the zCDP privacy budget for estimating a noisy gradient. If $\Delta_2(\grad f)$ is the $L_2$ sensitivity of the gradient of $f$ then, under zCDP we can measure the noisy gradient as $S_t=\grad f(\vec{w}_t) + N(\vec{0}, \frac{\Delta_2(\grad f)^2}{2\rho_t})$.

If our algorithm decides that this is not accurate enough, it will trigger a larger share of privacy budget $\rho_{t+1}>\rho_t$ to be applied at the next iteration. However, instead of discarding $S_t$, we perform another independent measurement using $\rho_{t+1}-\rho_t$ privacy budget: $S_t^\prime = \grad f(\vec{w}_t) + N(\vec{0}, \frac{\Delta_2(\grad f)^2}{2(\rho_{t+1}-\rho_t)})$.

We combine $S_t$ and $S^\prime_t$ in the following way:
$$\hat{S}_t = \frac{\rho_t S_t + (\rho_{t+1}-\rho_t)S^\prime_t}{\rho_t + (\rho_{t+1}-\rho_t)}$$
Simple calculations show that
\begin{align*}
E[\hat{S}_t]&=\grad f(\vec{w}_t)\\
Var(\hat{S}_t) &= \left(\rho_t^2\frac{\Delta_2(\grad f)^2}{2\rho_t} +  \frac{\Delta_2(\grad f)^2}{2(\rho_{t+1}-\rho_t)}(\rho_{t+1}-\rho_t)^2\right)/\rho_{t+1}^2\\
&= \frac{\Delta_2(\grad f)^2}{2\rho_{t+1}}
\end{align*}

Notice that computing $S_t$, then computing $S^\prime_t$ and obtaining the final estimate of the noisy gradient $\hat{S}_t$ uses a total privacy budget cost of $\rho_{t+1}$ and produces an answer with variance $\frac{\Delta_2(\grad f)^2}{2\rho_{t+1}}$. On the other hand, if we had magically known in advance that using a privacy budget share $\rho_t$ would lead to a bad gradient and pre-emptively used $\rho_{t+1}$ (instead of $\rho_t$) to measure the gradient, the privacy cost would be $\rho_{t+1}$ and the variance would still be $\frac{\Delta_2(\grad f)^2}{2\rho_{t+1}}$.

\section{Algorithm}
\label{sec:algorithm}

In this section, we provide a general framework for private ERM that
automatically adapts per-iteration privacy budget to make each
iteration progress toward an optimal solution. Let $D = \{d_1,
\ldots, d_n\}$ be an input database of $n$ independent
observations. Each observation $d_i = (\vec{x}_i, y_i)$ consists of
$\vec{x}_i \in \R^p$ and $y \in \R$. We consider empirical risk
minimization problem of the following form:
\begin{equation}
  \underset{\vec{w} \in \mathcal{C}}{\text{minimize}}\;
  f(\vec{w}; D) := \frac{1}{n}\sum_{i=1}^n \ell(\vec{w}; d_i)\,,
  \label{eq:erm}
\end{equation}
where $\ell$ is a loss function and $\mathcal{C}$ is a convex
set. Optionally, one may add a regularization term (e.g., 
$\frac{\lambda}{2}\norm{\vec{w}}_2^2$) 
into~\eqref{eq:erm} with no change in the privacy guarantee. Note that
the regularization term has no privacy implication as it is
independent of data. 

Algorithm~\ref{alg:dpagd} shows each step of the proposed
differentially private adaptive gradient descent algorithm
(DP-AGD). The algorithm has three main components: private
gradient approximation, step size selection, and adaptive noise
reduction. 
\begin{algorithm}[tp]
  \DontPrintSemicolon
  \SetNoFillComment
  \SetKwFunction{NDown}{NoiseDown}
   \SetKwFunction{Gavg}{GradAvg}
  \SetKwProg{Fn}{Function}{:}{end}
  \KwIn{privacy budget $\rho_{\mathsf{nmax}}, \rho_{\mathsf{ng}},
    \epsilon_{\mathsf{tot}}, \delta_{\mathsf{tot}}$,
    budget increase rate $\gamma$,
    clipping thresholds $C_{\mathsf{obj}}, C_{\mathsf{grad}}$,
    data $\{d_1,\dots,d_n\}$,
    objective function $f(\vec{w})=\sum_{i=1}^n \ell(\vec{w};d_i)$}
  Initialize $\vec{w}_0$ and $\Phi$\;
  $t\gets 0$, $\rho\gets$ solve~\eqref{eq:rho_conv} for $\rho$ \tcp{To compare to $(\epsilon,\delta)$-DP Algs} \label{line:conversion}
  \While{$\rho > 0$}{\label{line:check1}
    $i\gets 0$\;
    % $\sigma \gets (C_{\mathsf{grad}}/\epsilon_{\mathsf{ng}})\sqrt{2\ln(1.25/\delta_{\mathsf{tot}})}$\;
    $\vec{g}_t \gets \sum_{i=1}^n \left(\grad
    \ell(\vec{w}_t;{d}_i) / \max(1, \frac{\norm{\grad
    \ell(\vec{w}_t)}_2}{C_{\mathsf{grad}}})\right)$\;
    $\widetilde{\vec{g}}_t \gets \vec{g}_t + N(0, (C_{\mathsf{grad}}^2/2\rho_{\mathsf{ng}})\mat{I})$\;
    $\rho \gets \rho - \rho_{\mathsf{ng}}$\; \label{line:deduct_grad}
    $\widetilde{\vec{g}}_t \gets \widetilde{\vec{g}}_t /
    \norm{\widetilde{\vec{g}}_t}_2$\;
    % \tcc*{\footnotesize normalize $\widetilde{\vec{g}}_t$ to a unit vector}
    \While{$i = 0$}{
      $\Omega = \{f(\vec{w}_t-\alpha\widetilde{\vec{g}}_t): \alpha \in
      \Phi\}$\;
      $\rho \gets \rho - \rho_{\mathsf{nmax}}$\;
      $i \gets ${\scshape NoisyMax}(-$\Omega,
      C_{\mathsf{obj}}, \sqrt{2\rho_{\mathsf{nmax}}}$)\; \label{line:deduct_nmax}
      \eIf{$i > 0$}{
       % $\rho \gets \rho - \rho_{\mathsf{ng}}$\; \label{line:deduct_ndown}
        \lIf{$\rho > 0$}{$\vec{w}_{t+1} \gets \vec{w}_t - \alpha_i\widetilde{\vec{g}}_t$}\label{line:check2}
      }{
        $\rho_{\mathsf{old}} \gets \rho_{\mathsf{ng}}$\;
        $\rho_{\mathsf{ng}} \gets (1 + \gamma)\rho_{\mathsf{ng}}$\;
%        $\sigma_H \gets C_{\mathsf{grad}}/\sqrt{2\rho_{\mathsf{ng}}}$,         $\sigma_L \gets C_{\mathsf{grad}}/\sqrt{2\rho_{\mathsf{ng}}}$\;
        $\widetilde{\vec{g}}_t \gets \Gavg(\rho_{\mathsf{old}}, \rho_{\mathsf{ng}}, \vec{g}_t, \tilde{\vec{g}}_t, C_{\mathsf{grad}})$\;% sample from~\eqref{eq:ndown_dist}\;
        $\rho \gets \rho - (\rho_{\mathsf{ng}}-\rho_{\mathsf{old}})$\;\label{line:deduct_avg}
        % $\sigma \gets\sigma_L$\;
      }  % endif
    }  % end inner while
    $t \gets t + 1$\;
  }  % end outer while
  \Return $\vec{w}_t$\;
  \Fn{\Gavg{$\rho_{\mathsf{old}}$, $\rho_{H}$, $\vec{g}$, $\tilde{\vec{g}}$, $C_{\mathsf{grad}}$}}{
     $\tilde{\vec{g}}_2 \gets \vec{g} + N(\vec{0}, (\frac{C_{\mathsf{grad}}^2}{2(\rho_H-\rho_{\mathsf{old}})})\mat{I})$\;
     $\tilde{S}\gets \frac{\rho_{\mathsf{old}}\tilde{g} + (\rho_H-\rho_{\mathsf{old}})\tilde{g}_2}{\rho_H}  $\;
   \Return $\tilde{S}$
 }
  \caption{DP-AGD}
  \label{alg:dpagd}
\end{algorithm}
\paragraph{Gradient approximation}
At each iteration, the algorithm computes the noisy gradient
$\widetilde{\vec{g}} = \grad f(\vec{w}_t) + \mathcal{N}(0,
\sigma^2\mat{I})$ using the Gaussian mechanism with variance 
$\sigma^2$. The magnitude of noise $\sigma^2$ is dependent on the
maximum influence one individual can have on $\vec{g}_t$, measured by
$\Delta_2(g)$. To bound this quantity, many prior
works~\cite{Chaudhuri2011Objpert,Kifer2012erm} assume that
$\norm{\vec{x}} \leq 1$. Instead, we use the 
gradient clipping technique of~\cite{Abadi2016deep}: compute the
gradient $\grad \ell(\vec{w}_t; d_i)$ for $i=1, \ldots, n$, clip the
gradient in $L_2$ norm by dividing it by $\max(1, \frac{\norm{\grad
    \ell(\vec{w}_t;d_i)}_2}{C_{\mathsf{grad}}})$, compute the sum, add
Gaussian noise with variance
$C_{\mathsf{grad}}^2/2\rho_{\mathsf{ng}}$, and finally
normalize it to a unit norm. This ensures that the 
$L_2$ sensitivity of gradient is bounded by $C_{\mathsf{grad}}$, and
satisfies $\rho_{\mathsf{ng}}$-zCDP by
Lemma~\ref{lem:gauss2zcdp}. 

\paragraph{Step size selection}
% consider removing this paragraph
In non-private setting, stochastic optimization methods also use an
approximate gradient computed from a small set of randomly selected
data, called mini-batch, instead of an exact gradient. For example,
at iteration $t$, stochastic gradient descent (SGD) randomly picks an
index $i_t \in [n]$ and estimates the gradient $\grad \ell(\vec{w}_t;
d_{i_t})$ using one sample $d_{i_t}$. Consequently, each update
direction $-\grad \ell(\vec{w}_t;d_{i_t})$ might not be a descent
direction, but it is a descent direction in expectation since $\E[\grad
\ell(\vec{w}_t;d_{i_t})~|~\vec{w}_t] = \grad f(\vec{w}_t)$.

In contrast, in private setting an algorithm cannot rely on a
guarantee in expectation and need to use per-iteration 
privacy budget more efficiently. To best utilize the privacy budget,
we test whether a given a noisy estimate $\widetilde{\vec{g}}_t$ of
gradient is a descent direction using a portion of privacy budget
$\rho_{\mathsf{nmax}}$. First, the algorithm constructs a set
$\Omega=\{f(\vec{w}_t - \alpha\widetilde{\vec{g}}_t): \alpha \in
\Phi\}$, where each element of $\Omega$ is the objective value
evaluated at $\vec{w}_t - \alpha \widetilde{\vec{g}}_t$ and $\Phi$ is
the set of pre-defined step sizes. Then it determines which step size
yields the smallest objective value using the NoisyMax algorithm.
One difficulty in using the NoisyMax isgradient averaging
that there is no known a priori bound on a loss function $\ell$. To
bound the sensitivity of $\ell$, we apply the idea of 
gradient clipping to the objective function $f$. Given a fixed
clipping threshold $C_{\mathsf{obj}}$, we compute
$\ell(\vec{w}_t;\vec{x}_i)$ for $i=1, \ldots, n$, clip the values
greater than $C_{\mathsf{obj}}$, and take the summation of clipped
values. Note that, unlike the gradient, we use unnormalized value as
it doesn't affect the result of NoisyMax.

In our implementation, the first element of $\Phi$ is fixed to 0, so that
$\Omega$ always includes the current objective value $f(\vec{w}_t)$.
Let $i$ be the index returned by the NoisyMax. When $i>0$, the algorithm
updates $\vec{w}_t$ using the chosen step size $\alpha_i$. When $i =
0$, it is likely that $-\widetilde{\vec{g}}_t$ is not a descent
direction, and hence none of step sizes in $\Phi$ leads to a decrease
in objective function $f$.

\paragraph{Adaptive noise reduction}
When the direction $-\widetilde{\vec{g}}_t$ (obtained using the Gaussian
mechanism with parameter $\sigma$) is determined to be a bad
direction by the NoisyMax, DP-AGD increases the privacy budget for
noisy gradient approximation $\rho_{\mathsf{ng}}$ by a factor of
$1 + \gamma$. %, computes $\sigma_L (<\sigma)$ corresponding to the new
%increased $\rho_{\mathsf{ng}}$. 
Since the current gradient was measured using
the previous budget share, $\rho_{old}$, we use $\rho_{\mathsf{ng}} -\rho_{\mathsf{old}}$ privacy budget with the gradient averaging technique to increase the accuracy of that measured gradient. 
%
%, and then resamples a new direction
%from the distribution
%\begin{equation}
%  N\left(\frac{\sigma_L^2}{\sigma_H^2}\vec{y} + (1 -
%    \frac{\sigma_L^2}{\sigma_H^2})\vec{\mu}\,,
%    \frac{\sigma_L^2}{\sigma_H^2}(\sigma_H^2 - \sigma_L^2)\right)\,.
%  \label{eq:ndown_dist}
%\end{equation}
The new direction is
checked by the NoisyMax again. This procedure is repeated until the
NoisyMax finds a descent direction (i.e., until it returns a non-zero
index). 

%While this iterative resampling allows the algorithm to find a good
%search direction that can directly contribute to decrease the
%objective function $f$, a naive implementation could incur excessive
%privacy cost.
%
%Suppose it finds a descent direction after $K$ rounds.
%Let $\{\widetilde{\vec{g}}_t^{(k)}\}_{k=1}^K$ be the sequence of gradient and
%$\{\epsilon_{\mathsf{ng}}^{(k)}\}$ be their respective privacy
%budgets, where $\epsilon_{\mathsf{ng}}^{(k)} < \epsilon_{\mathsf{ng}}^{(k+1)}$
%for $k=1, 2, \ldots, K-1$. 
%In general, obtaining multiple estimates from a database
%requires the privacy budget of $\sum_{k=1}^K \epsilon_{\mathsf{ng}}^{(k)}$.
%To avoid this extra
%cost, we adopt the NoiseDown technique~\cite{Xiao2011iReduct}. The
%NoiseDown algorithm ensures that
%\[
%  \Pr[\vec{g}_t \given \widetilde{\vec{g}}_t^{(k)},
%  \widetilde{\vec{g}}_t^{(k+1)}] = \Pr[\vec{g}_t \given \widetilde{\vec{g}}_t^{(k+1)}]\,,
%\]
%where $\vec{g}_t$ is the non-private (true) gradient for iteration
%$t$. Thus, given the last (private) answer
%$\widetilde{\vec{g}}_t^{(K)}$, previous answers provide \emph{no}
%information about $\vec{g}_t$. 

\paragraph{Composition}
The two main tools used in DP-AGD achieves different versions of
differential privacy;  NoisyMax satisfies $(\epsilon, 0)$-DP (pure) and $\frac{\epsilon^2}{2}$-zCDP, while
Gaussian mechanism can be used to provide zCDP. However, to compare our method to other algorithms that use approximate $(\epsilon,\delta)$-differential privacy, we need to use conversion tools given by Lemmas \ref{lem:eps2zcdp} and \ref{lem:zcdp2dp}.

 Given the fixed total privacy budget
$(\epsilon_{\mathsf{tot}}, \delta_{\mathsf{tot}})$, the algorithm
starts by converting $(\epsilon_{\mathsf{tot}},
\delta_{\mathsf{tot}})$-DP into $\rho$-zCDP using
Lemma~\ref{lem:zcdp2dp}.
This is done by solving the following inequality for $\rho$:
\begin{equation}
  \epsilon_{\mathsf{tot}} \geq \rho +
  2\sqrt{\rho\log(1/\delta_{\mathsf{tot}})}\,.
  \label{eq:rho_conv}
\end{equation}
Given the resulting total privacy budget $\rho$ for zCDP, the algorithm dynamically
computes and deducts the amount of required privacy budget
(lines~\ref{line:deduct_avg},~\ref{line:deduct_nmax}, \ref{line:deduct_grad}) whenever it 
needs an access to the database during the runtime, instead of
allocating them a priori. This guarantees that the entire run of
algorithm satisfies $\rho$-zCDP. 

\paragraph{Adjusting step sizes}
For two reasons, we dynamically adjust the range of step sizes in
$\Phi$.
First, the variance in private gradient estimates needs to be
controlled. The stochastic gradient descent algorithm
with constant step sizes in general does not guarantee convergence to
the optimum even for a well-behaving objective function (e.g.,
strongly convex).
To guarantee the convergence (in expectation), stochastic optimization
algorithms typically enforce the conditions $\sum_{t} \alpha_t =
\infty$ and $\sum_{t} \alpha_t^2 < \infty$ on their step
sizes~\cite{Robbins1951stochastic}, which ensures that the variance of
the updates reduces gradually near the optimum. Although we adaptively reduce
the magnitude of privacy noise using gradient averaging, it
still needs a way to effectively control the variance of the updates.
Second, in our algorithm, it is possible that $\widetilde{\vec{g}}_t$
is actually a descent direction but the NoisyMax fails to choose a
step size properly. This 
happens when the candidate step sizes in $\Phi$ are all large but the
algorithm can only make a small move (i.e., when the optimal step
size is smaller than all non-zero step sizes in $\Phi$).
To address these issues, we propose to monitor the step sizes chosen
by NoisyMax algorithm and adaptively control the range of step sizes
in $\Phi$.
We initialize $\Phi$ with equally spaced $m$ points between 0 and
$\alpha_{\mathsf{max}}$. At every $\tau$ iteration, we update
$\alpha_{\mathsf{max}} = (1 + \eta)\max(\alpha_t, \alpha_{t-1}, \ldots,
\alpha_{t-\tau+1})$, where $\alpha_t$ denotes the step size chosen at
iteration $t$.
\footnote{In our experiments, we set $m=20$, $\tau=10$, and $\eta=0.1$
  and initialize $\alpha_{\mathsf{max}}=2$.}
We empirically observe
that this allows DP-AGD to adaptively change the range of step sizes
based on the relative location of the current iterate to the optimum.

\paragraph{Correctness of Privacy}
The correctness of the algorithm depends on $\rho$-zCDP composition (Lemma \ref{lem:comp}) and accounting for the privacy cost of each primitive.
\begin{theorem}Algorithm \ref{alg:dpagd} satisfies $\rho$-zCDP and $(\epsilon_{\mathsf{tot}},\delta_{\mathsf{tot}})$-differential privacy.
\end{theorem}
\begin{proof}
If no zCDP privacy budget $\rho$ is given to the algorithm, then the algorithm expects values for $\epsilon_{\mathsf{tot}}$ and $\delta_{\mathsf{tot}}$. It then figures out, in Line \ref{line:conversion}, the proper value of $\rho$ such that $\rho$-zCDP also satisfies the weaker  $(\epsilon_{\mathsf{tot}},\delta_{\mathsf{tot}})$-differential privacy.

Then, the algorithm works in pure zCDP mode, subtracting from its running budget the cost incurred by its three primitive operations: measuring the noisy gradient (Line \ref{line:deduct_grad}), performing a noisy max (Line \ref{line:deduct_nmax}), and gradient averaging whenever necessary (Line \ref{line:deduct_avg}). 

There are two checks that make sure the remaining privacy budget $\rho$ is above $0$, Line \ref{line:check1} and Line \ref{line:check2}. The most important check is in Line \ref{line:check2} -- it makes sure that any time we are updating the weights, we have not exhausted our privacy budget. This is the important step because the weights $\vec{w}_t$ are the only results that become visible outside the algorithm. For example, suppose the weights $\vec{w}_{t+1}$ were updated and the remaining privacy budget $\rho$ is greater than $0$, then it would be safe to release $\vec{w}_{t+1}$. Let us suppose now that the algorithm continues to the next iterations, but the deduction from measuring the noisy gradient, or NoisyMax or GradAvg causes the privacy budget $\rho$ to be negative. In this case, when we finally get to Line \ref{line:check2}, we do not update the weights, so we end up discarding the results of all of the primitives that were performed after the safe value for $\vec{w}_{t+1}$ had been computed. Subsequently, Line \ref{line:check1} will cause the algorithm to terminate. The result would be $\vec{w}_{t+1}$ which was already safe to release.

As these primitive operations use the correct share of the privacy budget they are given, the overall algorithm satisfies $\rho$-zCDP.
\end{proof}

%%% Local Variables:
%%% mode: latex
%%% TeX-master: "main"
%%% End:

\section{Experimental Results}
\label{sec:experiments}
In this section, we evaluate the performance of DP-AGD on 5 real
datasets:
\begin{inparaenum}[(i)]
\item \texttt{Adult}~\cite{Chang2011libsvm,Lichman2013UCI} dataset contains 48,842
  records of individuals from 1994 US Census.
\item \texttt{BANK}~\cite{Lichman2013UCI} contains marketing campaign
  related information about customers of a Portuguese banking institution.
\item \texttt{IPUMS-BR} and
\item \texttt{IPUMS-US} datasets are also Census data extracted from
  IPUMS-International~\cite{IPUMS}, and they contain 38,000 and 40,000
  records, respectively.
\item \texttt{KDDCup99} dataset contains attributes extracted from
  simulated network packets.
\end{inparaenum}
\begin{table}[tp]
  \begin{tabular}{cccl}
    \toprule
    Dataset & Size ($n$) & Dime. & Label \\ \midrule
    Adult & 48,842 & 124 & Is annual income > 50k? \\
    BANK & 45,211 & 33 & Is the product subscribed? \\
    IPUMS-US & 40,000 & 58 & Is annual income > 25k?\\
    IPUMS-BR & 38,000 & 53 & Is monthly income > \$300\\
    KDDCup99 & 4,898,431 & 120 & Is it a DOS attack? \\
    \bottomrule
  \end{tabular}
  \caption{Characteristics of datasets}
  \label{tab:datasets}
\end{table}
Table~\ref{tab:datasets} summarizes the characteristics of datasets
used in our experiments. 

\paragraph{Baselines}
We compare DP-AGD~\footnote{Python code for our experiments is available at \url{https://github.com/ppmlguy/DP-AGD}.}
against seven baseline algorithms, namely, 
\textsf{ObjPert}~\cite{Kifer2012erm,Chaudhuri2011Objpert},
\textsf{OutPert}~\cite{Zhang2017RR},
\textsf{PrivGene}~\cite{Zhang2013privgene},
\textsf{SGD-Adv}~\cite{Bassily2014PERM},
\textsf{SGD-MA}~\cite{Abadi2016deep},
\textsf{NonPrivate}, and \textsf{Majority}.
\textsf{ObjPert} is an objective perturbation method that adds a
linear perturbation term to the objective function.
\textsf{OutPert} is an output perturbation method that runs the
(non-private) batch gradient descent algorithm for a fixed number of
steps and then releases an output perturbed with Gaussian noise.
\textsf{PrivGene} is a differentially private genetic algorithm-based
model fitting framework. 
\textsf{SGD-Adv} is a differentially private version of SGD algorithm
that applies advanced composition theorem together with privacy
amplification result.
\textsf{SGD-MA} is also a private SGD algorithm but it uses an
improved composition method, called moments accountant, tailored to
the Gaussian noise distribution.
\textsf{NonPrivate} is an optimization algorithm that does not satisfy
differential privacy. In our experiments, to get the classification
accuracy of non-private method, we used the L-BFGS
algorithm~\cite{Nocedal1980updating}.
Finally, \textsf{Majority} predicts the
label by choosing the class with larger count. For example, if the
number of tuples with the label $y_i=1$ in the training set is greater
than $n/2$, it predicts that $y_i=1$ for all tuples.

In our experiments, we report both classification accuracy (i.e., the
fraction of correctly classified examples in the test set) and final
objective value (i.e., the value of $f$ at the last iteration). All
the reported numbers are 
averaged values over 20 times repeated 5-fold cross-validation.

\paragraph{Parameter settings}
When there are known default parameter settings for the prior works, we
used the same settings. Throughout all the experiments the value of
privacy parameter $\delta$ is fixed to $10^{-8}$ for the \texttt{Adult},
\texttt{BANK}, \texttt{IPUMS-US}, and \texttt{IPUMS-BR} datasets and to
$10^{-12}$ for the \texttt{KDDCup99} dataset. According to the common
practice in optimization, the sizes of mini-batches for SGD-ADV and SGD-MA are set to
$\sqrt{n}$. Since \texttt{KDDCup99} dataset contains approximately 5
million examples, directly computing gradients using the entire
dataset requires unacceptable computation time. To reduce the
computation time, at each iteration, we evaluate 
the gradient from a random subset of data (called mini-batch). For
\texttt{KDDCup99} dataset, we applied the mini-batch technique to both DP-AGD
and PrivGene, and fixed the mini-batch size to
40,000. We exclude \textsf{OutPert} from the experiments on
\texttt{KDDCup99} dataset as its sensitivity analysis requires using
full-batch gradient.

\textsf{OutPert} method has multiple parameters that
significantly affect its performance. For example, the number of
iterations $T$ and step size $\alpha$. We heavily tuned the parameters
$T$ and $\alpha$ by running \textsf{OutPert} on each
dataset with varying $T$ and $\alpha$, and chose the one that led to
the best performance.

To determine the regularization coefficient values, we find a fine-tuned
value of $\lambda$ using an L-BFGS algorithm with grid search. 
For SGD-Adv and SGD-MA, we used different coefficient values as these
two algorithms use mini-batched gradients.

% preprocessing
\subsection{Preprocessing}
Since all the datasets used in our experiments contain both numerical
and categorical attributes, 
we applied the following common preprocessing operations in machine
learning practice. We transformed every categorical attribute
into a set of binary variables by creating one binary variable for
each distinct category (i.e., one-hot encoding), and then every numerical
attribute is rescaled into the range $[0, 1]$ to ensure that all
attributes have the same scale. 
Additionally, for \textsf{ObjPert}, we normalize each observation
to a unit norm (i.e., $\norm{\vec{x}_i}_2 = 1$ for $i=1, 2, \ldots,
n$) to satisfy its requirement.

\subsection{Effects of Parameters}
We first demonstrate the impact of internal parameter settings on the
performance of DP-AGD. For this set of experiments, we perform a
logistic regression task on the \texttt{Adult} dataset.
The proposed algorithm has several internal parameters, 
and in this set of experiments we show that in general its performance
is relatively robust to their settings.
\begin{figure*}[tp]
  \centering
  \begin{subfigure}[t]{0.49\textwidth}
    \centering
    \includegraphics[width=.49\textwidth]{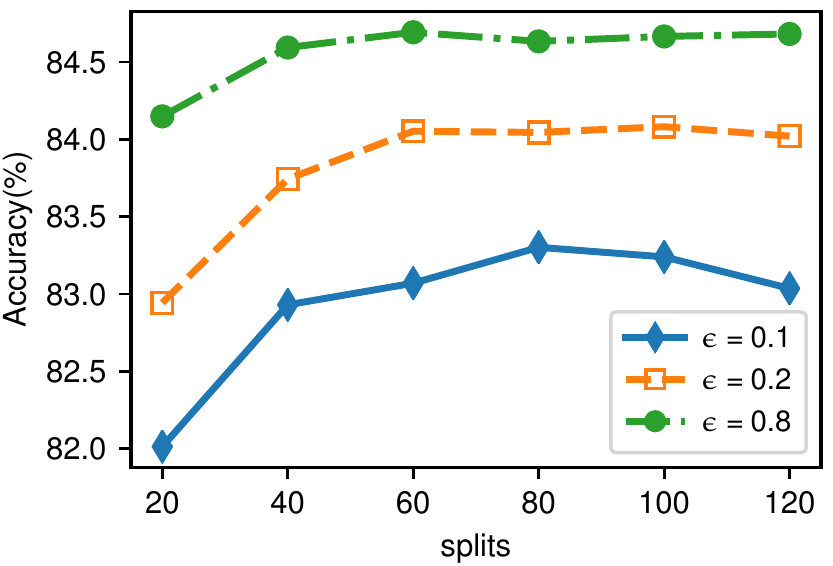}
    \includegraphics[width=.49\textwidth]{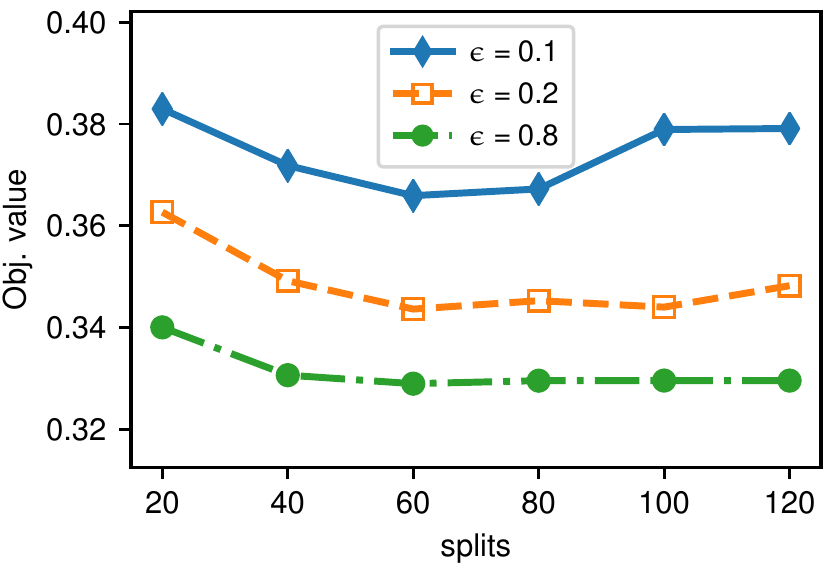}
    \caption{Effect of \textsf{splits} (left: accuracy, right: obj.
      value)}
    \label{fig:splits}
  \end{subfigure}
  \begin{subfigure}[t]{0.49\textwidth}
    \centering
    \includegraphics[width=.49\textwidth]{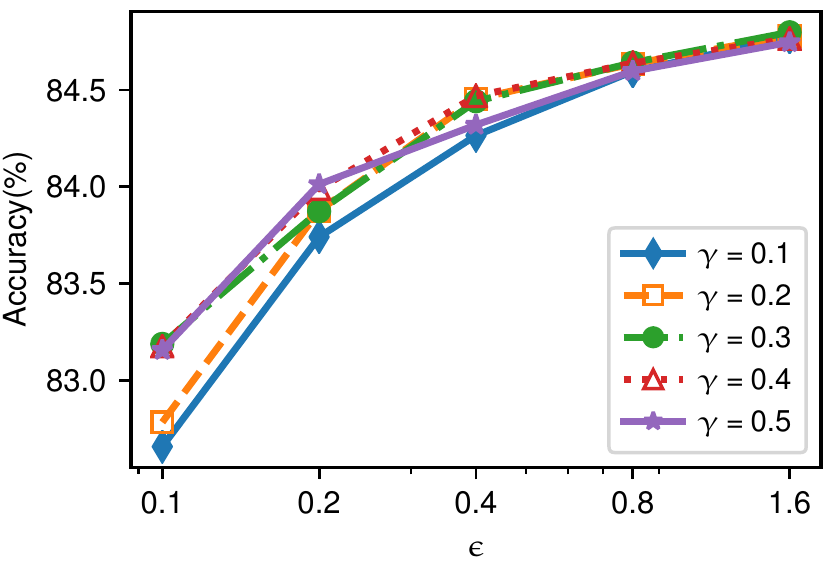}
    \includegraphics[width=.49\textwidth]{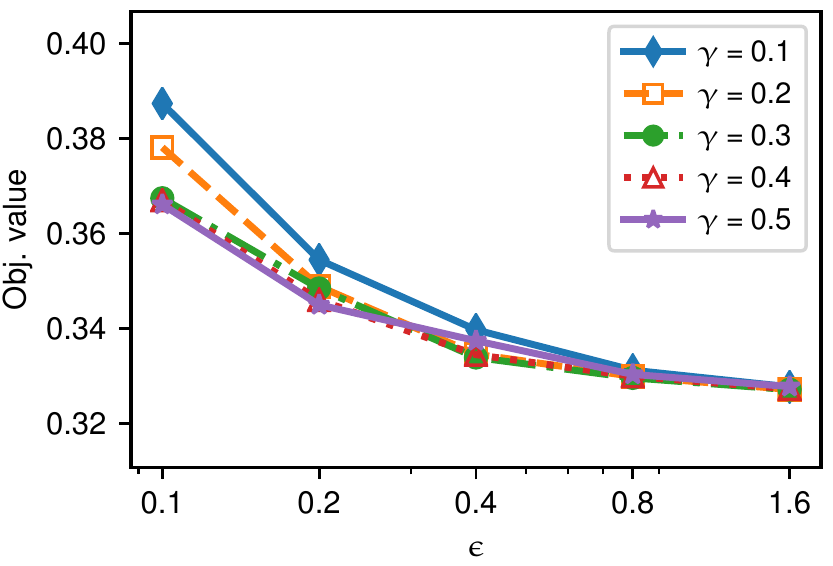}
    \caption{Effect of $\gamma$ (left: accuracy, right: obj. value)}
    \label{fig:gamma}
  \end{subfigure}~\\
  \begin{subfigure}[t]{0.49\textwidth}
    \centering
    \includegraphics[width=.49\textwidth]{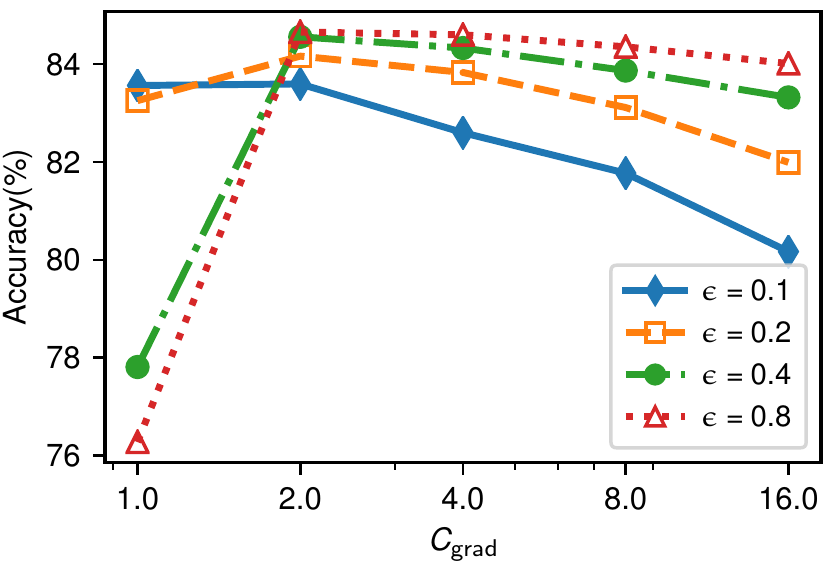}
    \includegraphics[width=.49\textwidth]{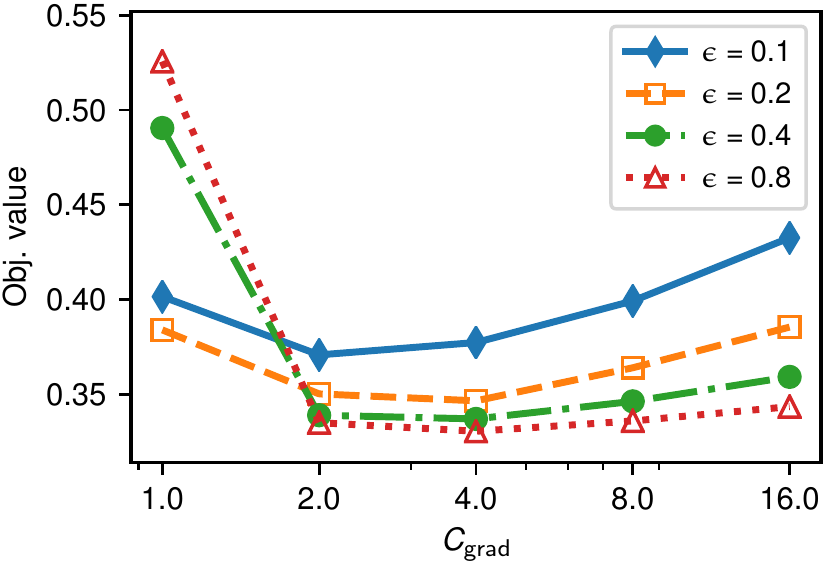}
    \caption{Effect of $C_{\mathsf{grad}}$ (left: accuracy, right: obj. value)}
    \label{fig:cgrad}
  \end{subfigure}
  \begin{subfigure}[t]{0.49\textwidth}
    \centering
    \includegraphics[width=.49\textwidth]{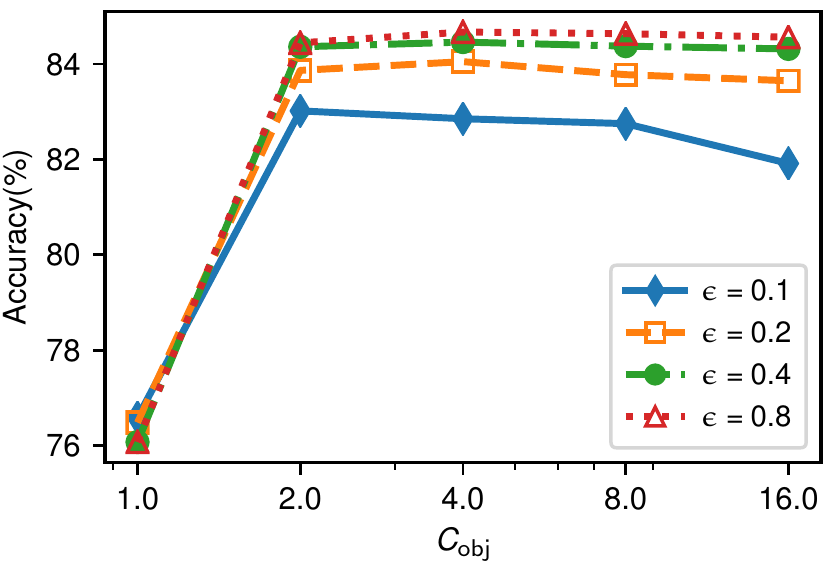}
    \includegraphics[width=.49\textwidth]{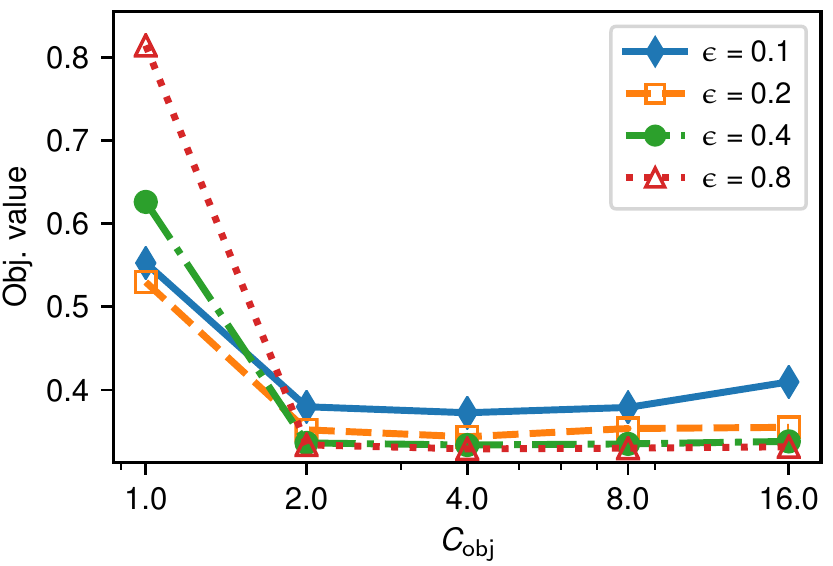}
    \caption{Effect of $C_{\mathsf{obj}}$ (left: accuracy, right: obj. value)}
    \label{fig:cobj}
  \end{subfigure}
  \caption{Effects of parameters}
\end{figure*}

In our experiments, to determine the initial privacy budget
parameters $\rho_{\mathsf{nmax}}$ and $\rho_{\mathsf{ng}}$, we first
compute their corresponding $\epsilon_{\mathsf{namx}}$ and
$\epsilon_{\mathsf{ng}}$ values as follows:
\[
  \epsilon_{\mathsf{nmax}} =\epsilon_{\mathsf{ng}} =
  \frac{\epsilon_{\mathsf{tot}}}{2\cdot \mathsf{splits}}\,. 
\]
Then these values are converted back to $\rho_{\mathsf{nmanx}}$ and
$\rho_{\mathsf{ng}}$, respectively, using Lemma~\ref{lem:gauss2zcdp}
and~\ref{lem:eps2zcdp}. 
The intuition behind this setting is that the value of \textsf{splits}
roughly represents the number of iterations under the naive (linear)
composition.
We fixed \textsf{splits}=60 for all
experiments. Figure~\ref{fig:splits} shows the impacts of
\textsf{splits} on the algorithm's performance. From the figure, we
see the performance of DP-AGD is relatively less affected by the choice of
\textsf{splits} when $\epsilon$ is large. When $\epsilon=0.1$,
excessively small or large value of \textsf{splits} can degrade the
performance. When \textsf{splits} is set too small, the algorithm may
not have enough number of iterations to converge to the optimum. On
the other hand, when the value of \textsf{splits} is too large, the
algorithm may find it difficult to discover good search directions.

Figure~\ref{fig:gamma} describes how classification accuracy
and objective value change with varying values of $\gamma$. The parameter
$\gamma$ controls how fast the algorithm increases
$\epsilon_{\mathsf{ng}}$ when the gradient for the current iteration
is not a descent direction. As it can be seen from the figure, when
$\gamma > 0.2$, the value of $\gamma$ almost has no impact on the
performance. On the other hand, when $\gamma=0.1$ or $\gamma=0.2$, the
performance is slightly affected by the setting. This is because, when
$\gamma$ is too small, the new estimate of gradient might be still too
noisy, and as a result the algorithm spends more privacy budget on
executing NoisyMax.

The algorithm has two threshold parameters, $C_{\mathsf{grad}}$ and
$C_{\mathsf{obj}}$. These two parameters are used to bound the
sensitivity of gradient and objective value computation,
respectively. If these parameters are set to a too small value, it
significantly reduces the sensitivity but at the same time it can
cause too much information loss in the estimates. Conversely, if they
are set too high, the sensitivity becomes high, resulting in adding
too much noise to the estimates. In our experiments, both
$C_{\mathsf{obj}}$ and $C_{\mathsf{grad}}$ are fixed to 3.0.

In Figure~\ref{fig:cgrad}, to see the impact of $C_{\mathsf{grad}}$ on
the performance, we fix the objective clipping threshold
$C_{\mathsf{obj}}$ to our default value 3.0 and vary
$C_{\mathsf{grad}}$ from 1.0 to 16.0. The figure illustrates how the
accuracy and the final objective value change with varying values of
$C_{\mathsf{grad}}$. We observe that excessively large or small
threshold values can degrade the performance. 
As explained above, this is because of the
trade-off between high sensitivity and information loss.

Figure~\ref{fig:cobj} shows how $C_{\mathsf{obj}}$ affects the
accuracy and the final objective value. As it was the case for
$C_{\mathsf{grad}}$, too large or small values have a negative effect
on the performance, while the moderate values
($2\leq C_{\mathsf{obj}} \leq 8$) have little impact on the
performance. 

\subsection{Logistic Regression and SVM}
We applied our DP-AGD algorithm to a regularized logistic regression
model in which the goal is
\[
  \begin{aligned}
    & \underset{\vec{w}}{\min} & &
  \frac{1}{n}\sum_{i=1}^n \log(1 +
  \exp(-y_i\vec{w}^\intercal\vec{x}_i)) + \frac{\lambda}{2}\norm{\vec{w}}_2^2\,,
  \end{aligned}
\]
where $\vec{x}_i \in \R^{p+1}$, $y_i \in \{-1, +1\}$, and $\lambda >
0$ is a regularization coefficient.

\begin{figure*}[htp]
  \centering
  \begin{subfigure}[t]{\textwidth}
    \centering
    \includegraphics[width=7in]{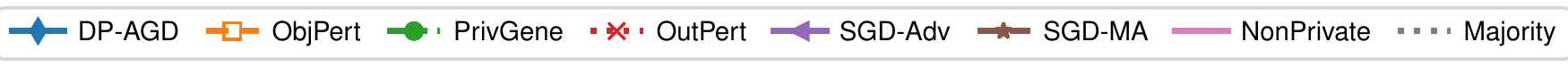}    
  \end{subfigure}~\\
  \begin{subfigure}[t]{0.24\textwidth}
    \centering
    \includegraphics[width=\textwidth]{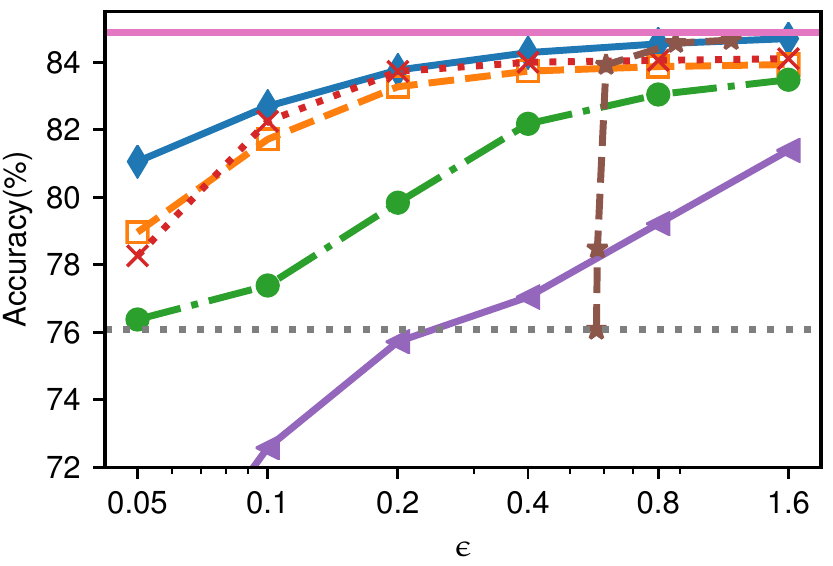}
  \end{subfigure}
  \begin{subfigure}[t]{0.24\textwidth}
    \centering
    \includegraphics[width=\textwidth]{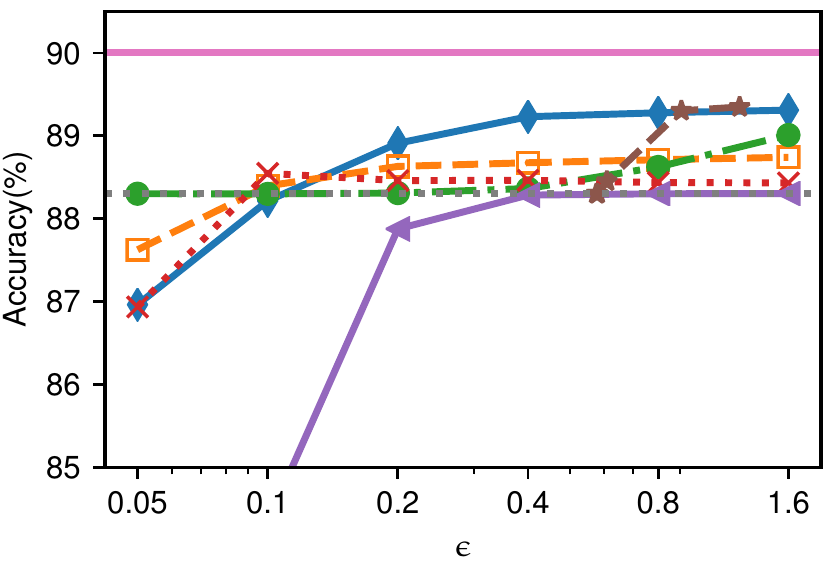}
  \end{subfigure}
    \begin{subfigure}[t]{0.24\textwidth}
    \centering
    \includegraphics[width=\textwidth]{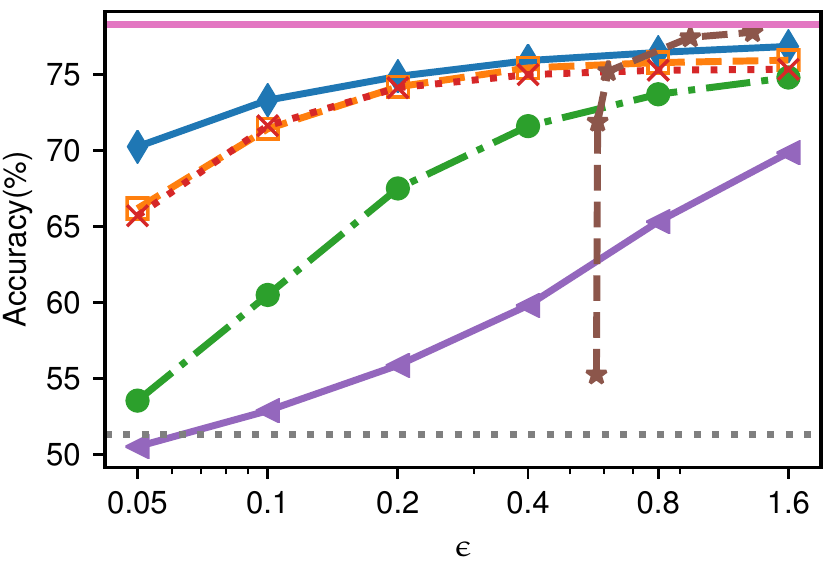}    
  \end{subfigure}
    \begin{subfigure}[t]{0.24\textwidth}
    \centering
    \includegraphics[width=\textwidth]{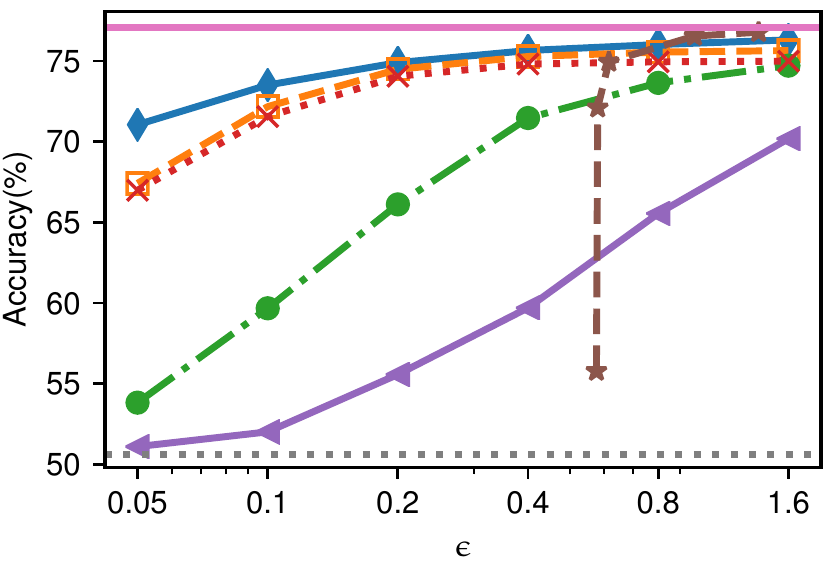}
  \end{subfigure}~\\
    \begin{subfigure}[t]{0.24\textwidth}
    \centering
    \includegraphics[width=\textwidth]{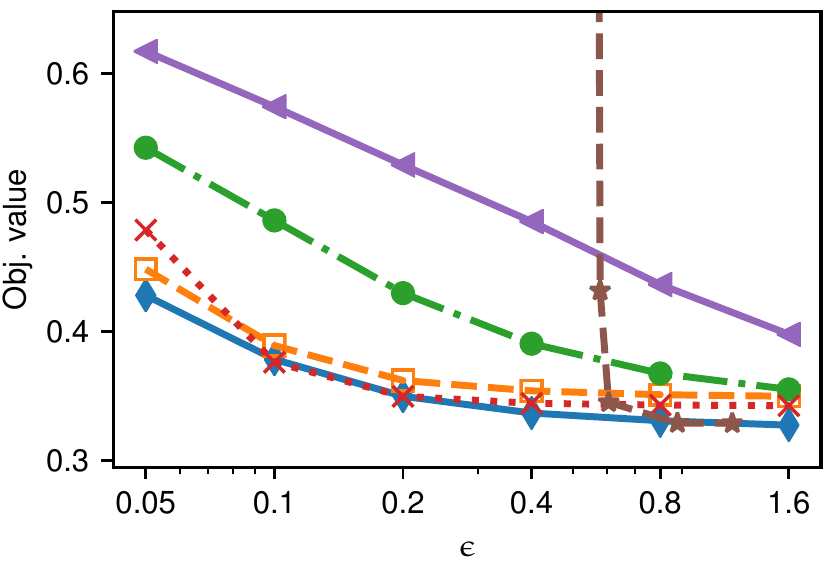}
    \caption{Adult}
  \end{subfigure}
  \begin{subfigure}[t]{0.24\textwidth}
    \centering
    \includegraphics[width=\textwidth]{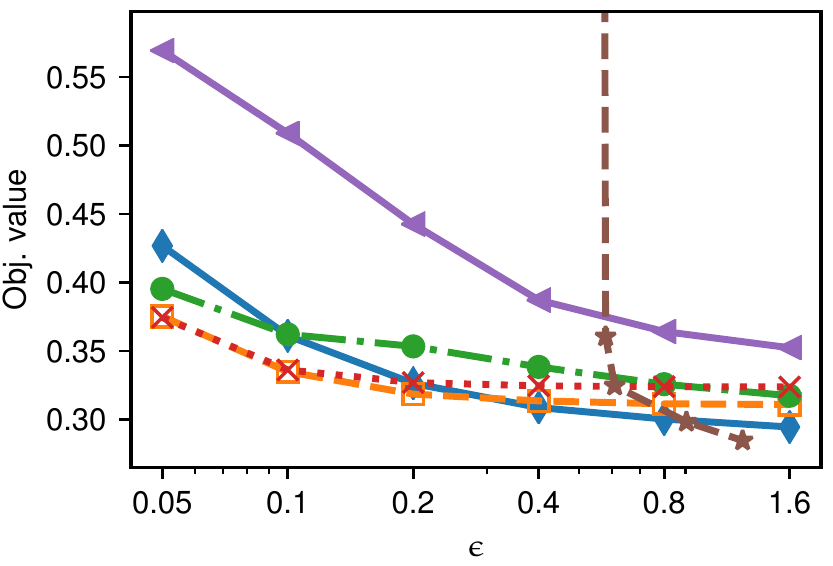}
    \caption{BANK}
  \end{subfigure}
    \begin{subfigure}[t]{0.24\textwidth}
    \centering
    \includegraphics[width=\textwidth]{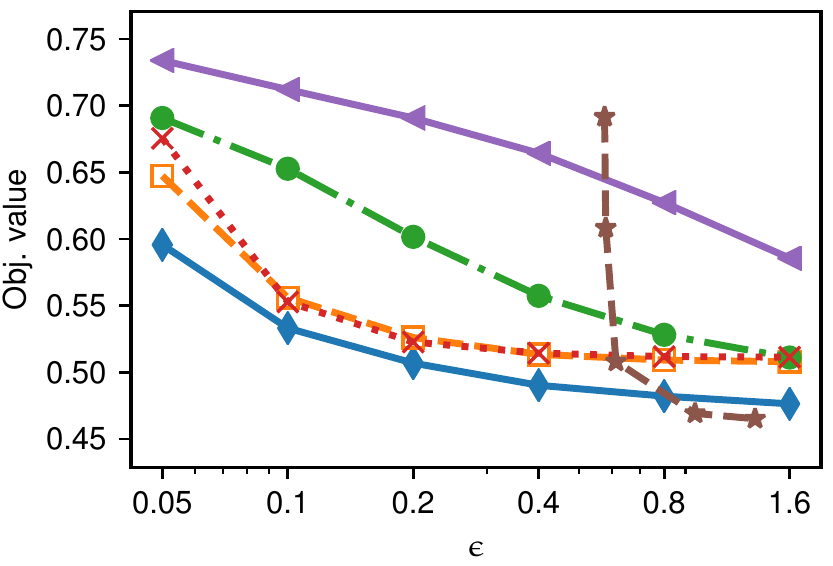}
    \caption{IPUMS-US}
  \end{subfigure}
    \begin{subfigure}[t]{0.24\textwidth}
    \centering
    \includegraphics[width=\textwidth]{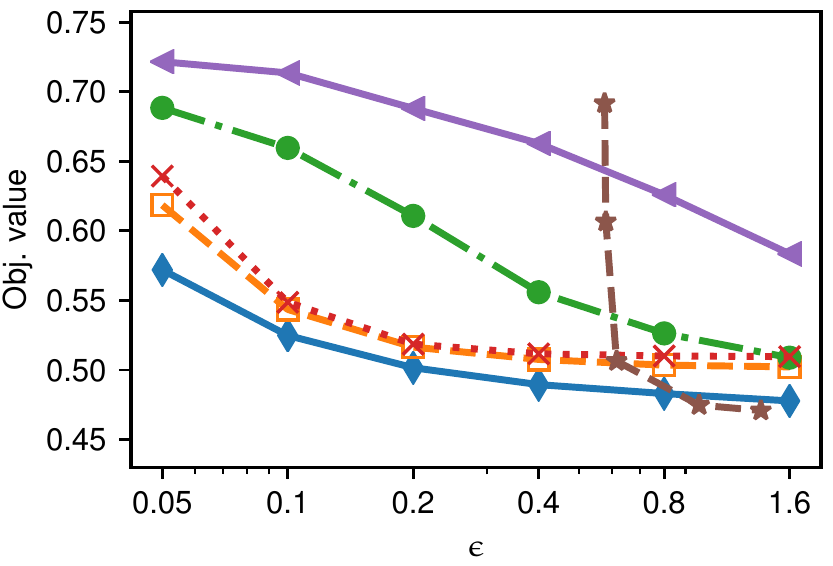}
    \caption{IPUMS-BR}
  \end{subfigure}
  \caption{Logistic regression by varying $\epsilon$ (Top:
    classification accuracies, Bottom: objective values)}
  \label{fig:logreg_4}
\end{figure*}
\begin{figure*}[tp]
  \centering
  \begin{subfigure}[t]{\textwidth}
    \centering
    \includegraphics[width=7in]{legend_cmp_all}    
  \end{subfigure}~\\
  \begin{subfigure}[t]{0.24\textwidth}
    \centering
    \includegraphics[width=\textwidth]{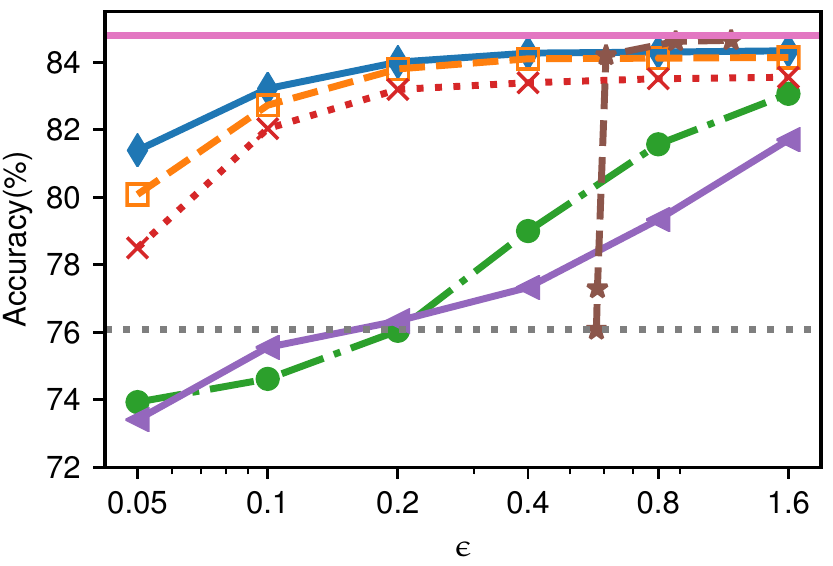}
  \end{subfigure}
  \begin{subfigure}[t]{0.24\textwidth}
    \centering
    \includegraphics[width=\textwidth]{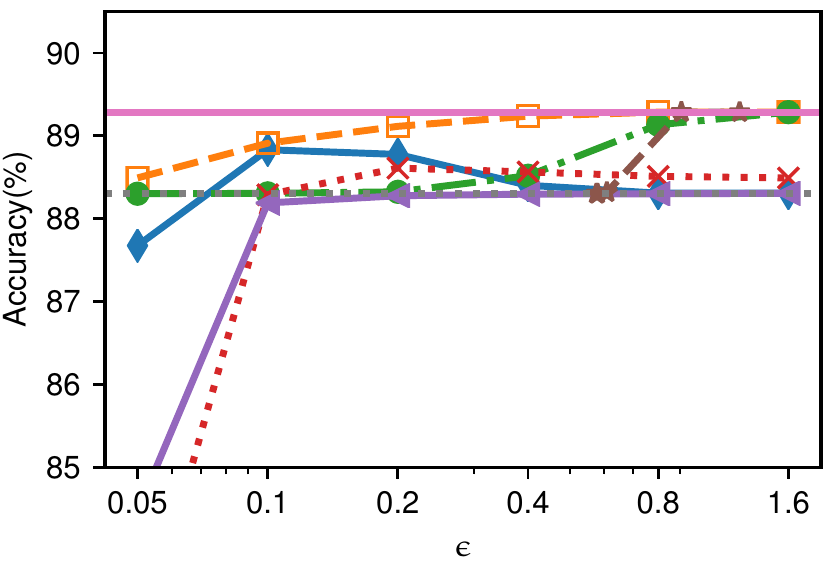}
  \end{subfigure}
    \begin{subfigure}[t]{0.24\textwidth}
    \centering
    \includegraphics[width=\textwidth]{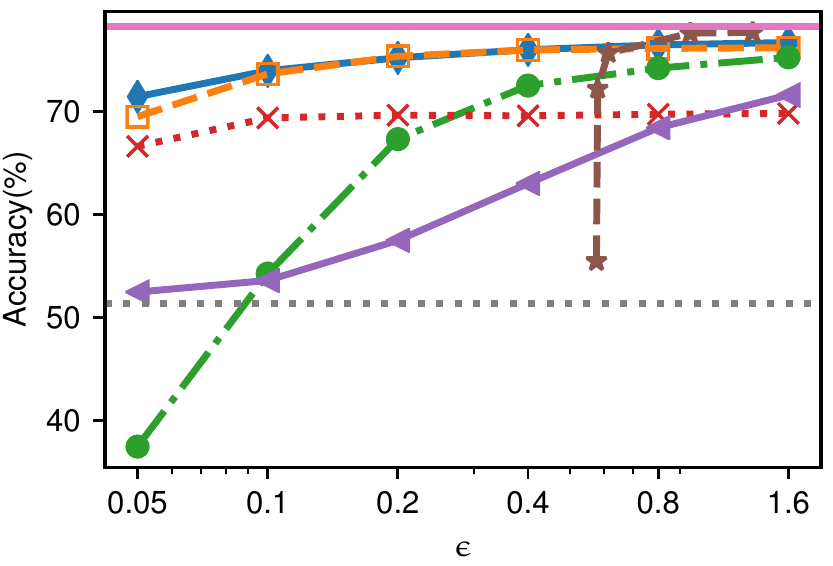}    
  \end{subfigure}
    \begin{subfigure}[t]{0.24\textwidth}
    \centering
    \includegraphics[width=\textwidth]{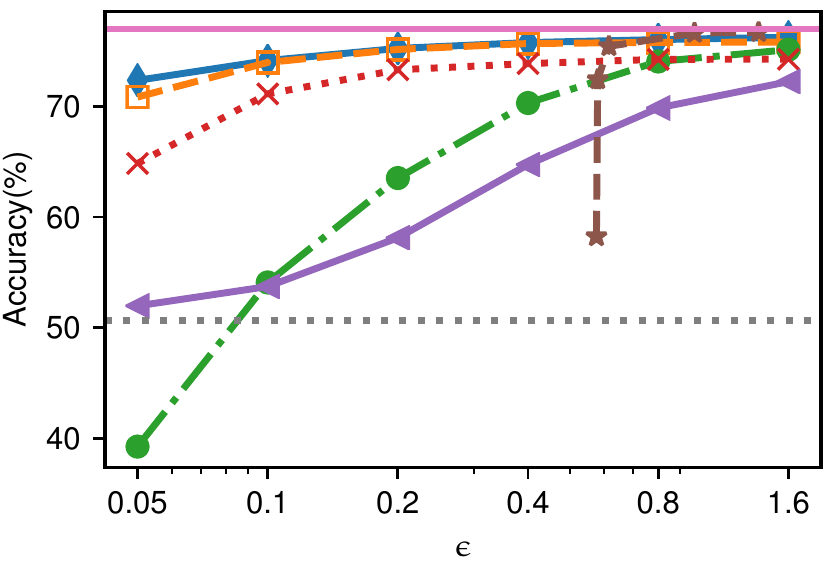}
  \end{subfigure}~\\
    \begin{subfigure}[t]{0.24\textwidth}
    \centering
    \includegraphics[width=\textwidth]{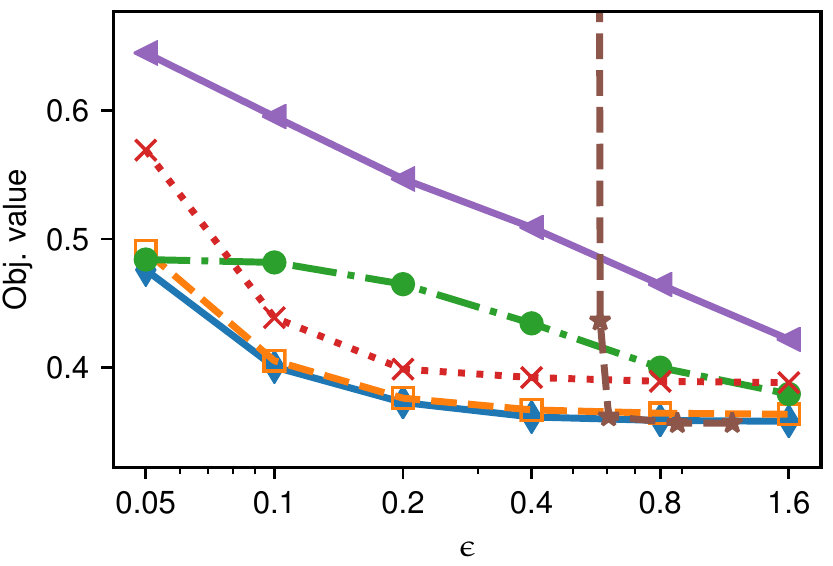}
    \caption{Adult}
  \end{subfigure}
  \begin{subfigure}[t]{0.24\textwidth}
    \centering
    \includegraphics[width=\textwidth]{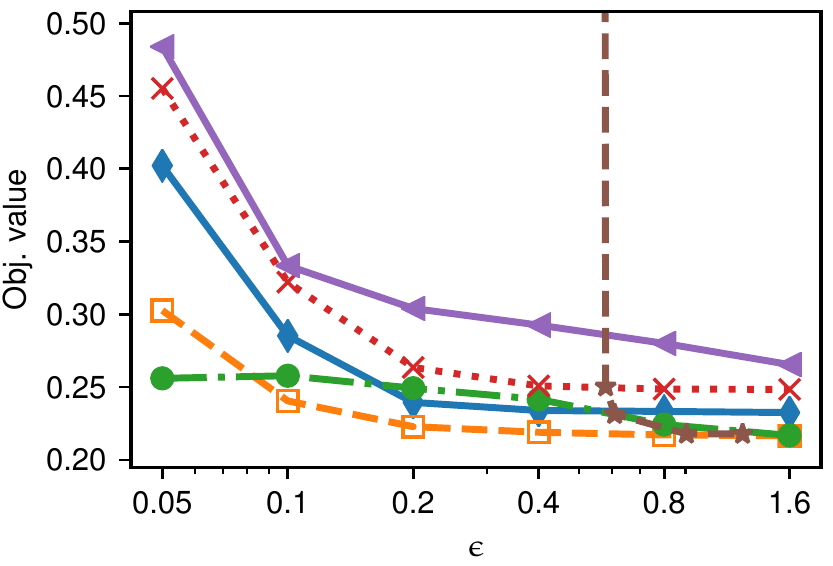}
    \caption{BANK}
  \end{subfigure}
    \begin{subfigure}[t]{0.24\textwidth}
    \centering
    \includegraphics[width=\textwidth]{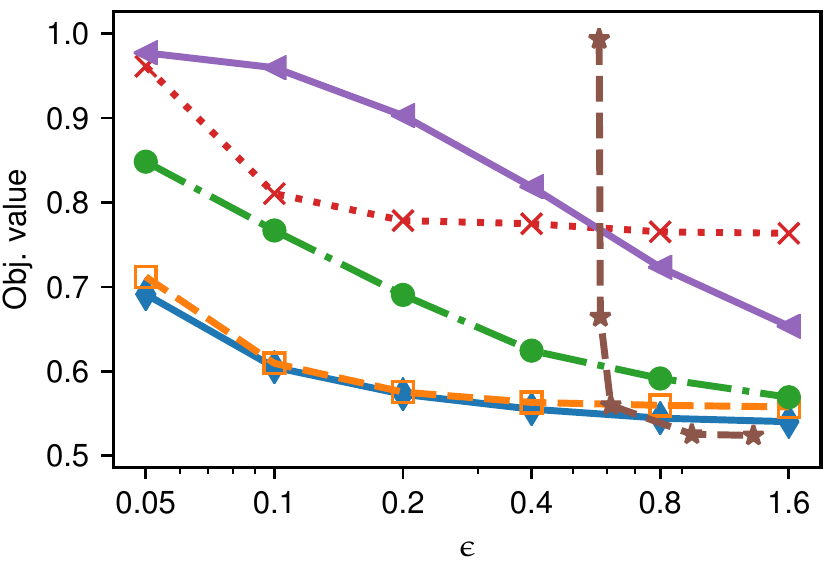}
    \caption{IPUMS-US}
  \end{subfigure}
    \begin{subfigure}[t]{0.24\textwidth}
    \centering
    \includegraphics[width=\textwidth]{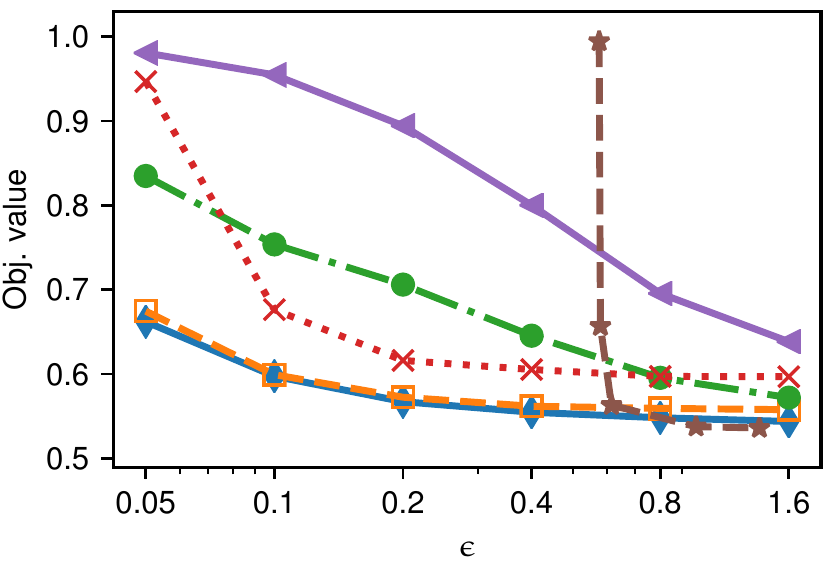}
    \caption{IPUMS-BR}
  \end{subfigure}
  \caption{SVM by varying $\epsilon$ (Top:
    classification accuracies, Bottom: objective values)}
  \label{fig:svm_4}
\end{figure*}
\begin{figure*}[tp]
  \centering
  \begin{subfigure}[t]{\textwidth}
    \centering
    \includegraphics[width=7in]{legend_cmp_all}    
  \end{subfigure}~\\
  \begin{subfigure}[t]{.49\textwidth}
    \centering
    \includegraphics[width=.49\textwidth]{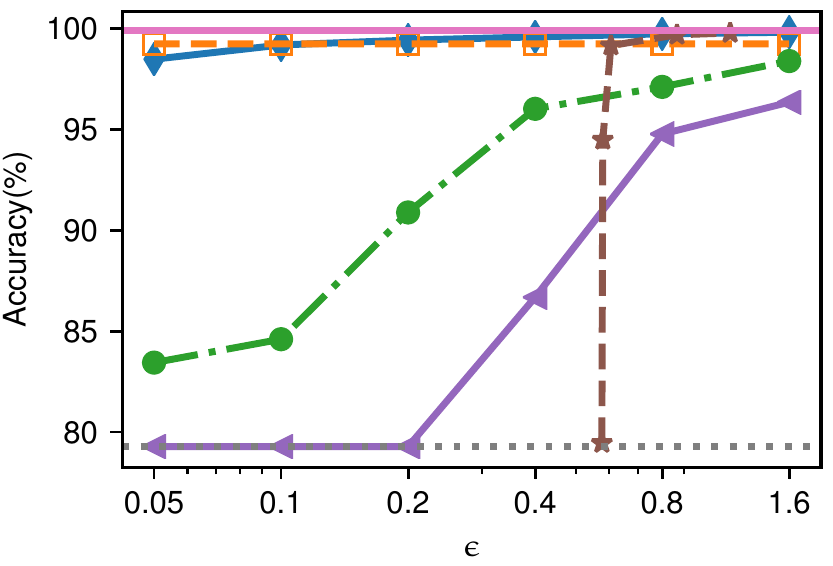}
    \includegraphics[width=.49\textwidth]{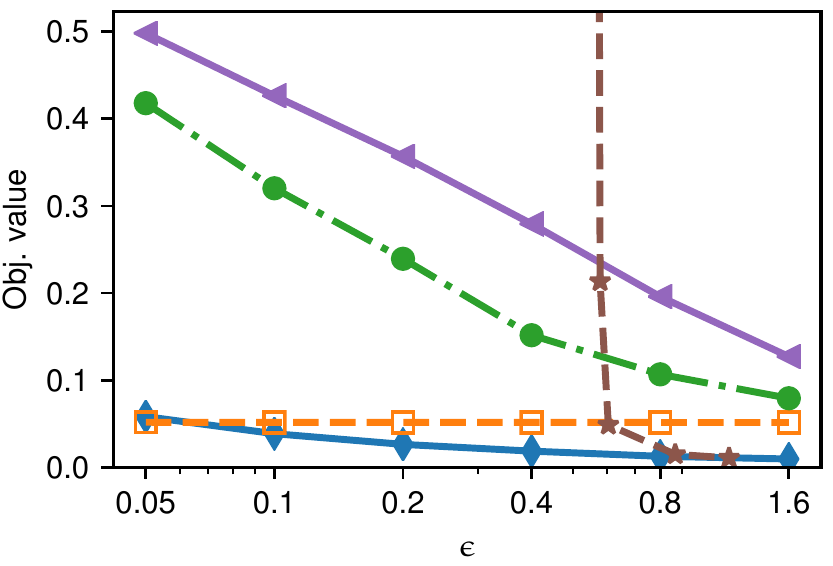}
  \end{subfigure}
  \begin{subfigure}[t]{.49\textwidth}
    \centering
    \includegraphics[width=.49\textwidth]{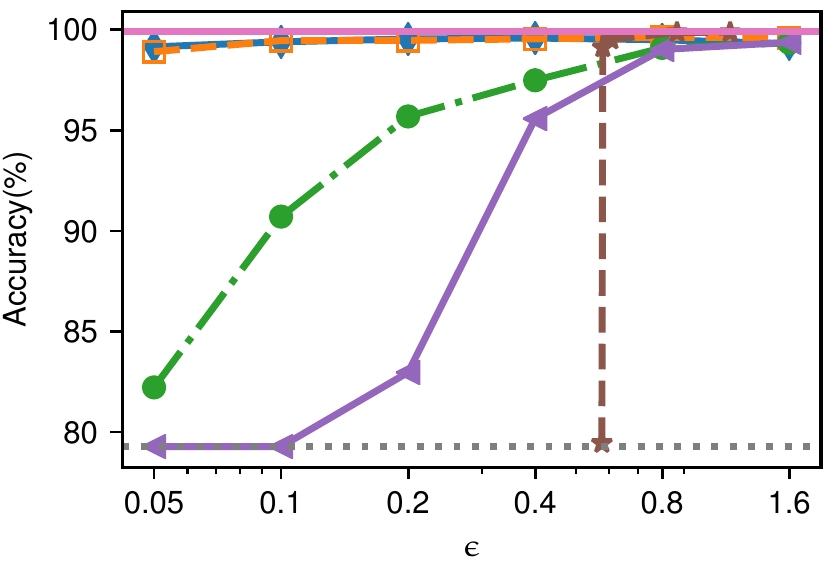}
    \includegraphics[width=.49\textwidth]{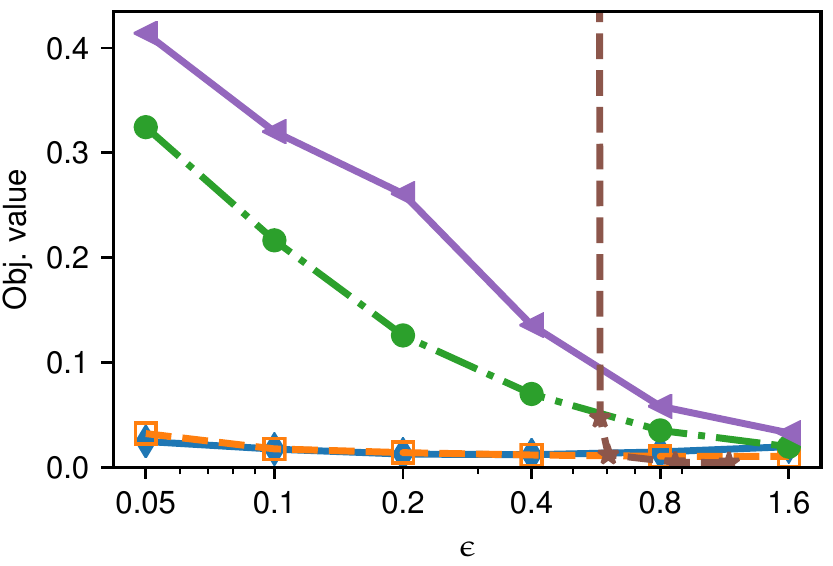}
  \end{subfigure}
  \caption{Classification task on \texttt{KDDCup99} dataset (Left:
    logistic regression, right: SVM)}
  \label{fig:kddcup99}
\end{figure*}

The top row in Figure~\ref{fig:logreg_4} shows the classification
accuracies of logistic regression model on
4 different datasets. The result on \texttt{KDDCup99} dataset is shown
in Figure~\ref{fig:kddcup99}. The proposed DP-AGD algorithm consistently
outperforms or performs 
competitively with other algorithms on a wide range of $\epsilon$
values. Especially, when $\epsilon$ is very small (e.g.,
$\epsilon=0.05$), DP-AGD outperforms all other methods except on the
BANK dataset. This is because other algorithms tend to waste privacy
budgets by obtaining extremely noisy statistics and
blindly use them in their updates without checking whether it can lead
to a better solution (i.e., they perform many updates that do not help
decrease the objective value). On the other hand, DP-AGD explicitly checks the
usefulness of the statistics and only use them when they can
contribute to decreasing the objective value.

The bottom row of Figure~\ref{fig:logreg_4} illustrates how the final
objective value achieved by each algorithm changes as the value of
$\epsilon$ increases. As it was observed in the experiments on
classification accuracy, DP-SGD gets close to the best achievable
objective values for a wide range of $\epsilon$ values considered in
the experiments.

It should be emphasized that the accuracies of SGD-Adv and SGD-MA are
largely dependent on the total number of iterations $T$. In all of the
baseline algorithms, the value of
$T$ needs to be determined before the execution of the algorithms. To
get the best 
accuracy for the given value of $\epsilon$, it requires tuning
the value of $T$ through multiple interactions with a
dataset (e.g., trial and error), which also should be done in a
differentially private manner and hence it requires a portion of
privacy budget. In PrivGene, the number of iterations is
heuristically set to $T = c\cdot (n \cdot \epsilon)$, where $c$ is a
tuning parameter and $n$ is the number of observations in
$D$. However, $T$ still requires a careful tuning as it depends on
$c$.

SGD-MA outperforms all other algorithms when $\epsilon > 0.8$. This is
because SGD-MA can afford more number of iterations resulting from tight
bound on the privacy loss provided by the moments accountant, together
with privacy amplification effect due to subsampling. However, it is
hard to use the moments accountant method under high privacy regime
(i.e., when $\epsilon$ is small) because the bound is not sharp when
there are small number of independent random variables (i.e., when the
number of iterations is small). With $\delta$ fixed to $10^{-8}$, we
empirically 
observe that, under the moments accountant, one single iteration of SGD update can incur the privacy
cost of $\epsilon\approx 0.5756$. This renders the moments accountant
method impractical when high level of privacy protection is required.

Support vector machine (SVM) is one of the most effective tools for
classification problems. In this work, we only consider the linear SVM
(without kernel) for simplicity. The SVM classification problem is
formulated as an optimization problem:
\[
  \begin{aligned}
    & \underset{\vec{w}}{\min} & &
    \frac{\lambda}{2}\norm{\vec{w}}_2^2 + \frac{1}{n} \sum_{i=1}^n \max\{1 -
    y_i\vec{w}^\intercal\vec{x}_i, 0\}\,,
  \end{aligned}
\]
where $\vec{x}_i \in \R^{p+1}$ and $y_i \in \{-1, + 1\}$ for $i \in [n]$.

Figure~\ref{fig:svm_4} compares the performance of DP-AGD on SVM task
with other baseline algorithms. As it was shown in the experiments on
logistic regression, DP-AGD achieves the competitive accuracies on a
wide range of values for $\epsilon$.

DP-AGD showed an unstable behavior when performing SVM task on
\texttt{BANK} dataset: its accuracy can degrade even though we use
more privacy budget. For example, the accuracy when $\epsilon=0.4$ is
lower than that when $\epsilon=0.1$. However, we observe that its
objective value consistently 
decreases as the value of $\epsilon$ is increased.

%%% Local Variables:
%%% mode: latex
%%% TeX-master: "main"
%%% End:

\section{Conclusion}
This paper has developed an iterative optimization algorithm for
differential privacy, in which the per-iteration privacy budget is
adaptively determined based on the utility of privacy-preserving
statistics.
% Iterative optimization algorithms for differential privacy repeatedly
% query database and receive a sequence of (noisy) answers. In order for
% the final solution to be useful, each (noisy) answer should provide
% useful information about the objective function so that it can be
% used to make improvement towards the optimum.
%
Existing private algorithms lack runtime adaptivity to account for
statistical utility of intermediate query answers. To address this
significant drawback, we presented a general framework for adaptive
privacy budget selection. While the proposed algorithm has been
demonstrated in the context of private ERM problem, we believe our
approach can be easily applied to other problems.

%%% Local Variables:
%%% mode: latex
%%% TeX-master: "main"
%%% End:

\bibliographystyle{ACM-Reference-Format}
\bibliography{reference}


\begin{thebibliography}{25}

%%% the Web address.

\ifx \showCODEN    \undefined \def \showCODEN     #1{\unskip}     \fi
\ifx \showDOI      \undefined \def \showDOI       #1{#1}\fi
\ifx \showISBNx    \undefined \def \showISBNx     #1{\unskip}     \fi
\ifx \showISBNxiii \undefined \def \showISBNxiii  #1{\unskip}     \fi
\ifx \showISSN     \undefined \def \showISSN      #1{\unskip}     \fi
\ifx \showLCCN     \undefined \def \showLCCN      #1{\unskip}     \fi
\ifx \shownote     \undefined \def \shownote      #1{#1}          \fi
\ifx \showarticletitle \undefined \def \showarticletitle #1{#1}   \fi
\ifx \showURL      \undefined \def \showURL       {\relax}        \fi
% invisible to TeX
\providecommand\bibfield[2]{#2}
\providecommand\bibinfo[2]{#2}
\providecommand\natexlab[1]{#1}
\providecommand\showeprint[2][]{arXiv:#2}

\bibitem[\protect\citeauthoryear{Abadi, Chu, Goodfellow, McMahan, Mironov,
  Talwar, and Zhang}{Abadi et~al\mbox{.}}{2016}]%
        {Abadi2016deep}
\bibfield{author}{\bibinfo{person}{Martin Abadi}, \bibinfo{person}{Andy Chu},
  \bibinfo{person}{Ian Goodfellow}, \bibinfo{person}{H~Brendan McMahan},
  \bibinfo{person}{Ilya Mironov}, \bibinfo{person}{Kunal Talwar}, {and}
  \bibinfo{person}{Li Zhang}.} \bibinfo{year}{2016}\natexlab{}.
\newblock \showarticletitle{Deep learning with differential privacy}. In
  \bibinfo{booktitle}{\emph{Proceedings of the 2016 ACM SIGSAC Conference on
  Computer and Communications Security}}. ACM, \bibinfo{pages}{308--318}.
\newblock


\bibitem[\protect\citeauthoryear{Bassily, Smith, and Thakurta}{Bassily
  et~al\mbox{.}}{2014}]%
        {Bassily2014PERM}
\bibfield{author}{\bibinfo{person}{Raef Bassily}, \bibinfo{person}{Adam Smith},
  {and} \bibinfo{person}{Abhradeep Thakurta}.} \bibinfo{year}{2014}\natexlab{}.
\newblock \showarticletitle{Private Empirical Risk Minimization: Efficient
  Algorithms and Tight Error Bounds}. In \bibinfo{booktitle}{\emph{Proceedings
  of the 2014 IEEE 55th Annual Symposium on Foundations of Computer Science}}
  \emph{(\bibinfo{series}{FOCS '14})}. \bibinfo{publisher}{IEEE Computer
  Society}, \bibinfo{address}{Washington, DC, USA}, \bibinfo{pages}{464--473}.
\newblock


\bibitem[\protect\citeauthoryear{Beimel, Brenner, Kasiviswanathan, and
  Nissim}{Beimel et~al\mbox{.}}{2014}]%
        {Beimel2014bounds}
\bibfield{author}{\bibinfo{person}{Amos Beimel}, \bibinfo{person}{Hai Brenner},
  \bibinfo{person}{Shiva~Prasad Kasiviswanathan}, {and} \bibinfo{person}{Kobbi
  Nissim}.} \bibinfo{year}{2014}\natexlab{}.
\newblock \showarticletitle{Bounds on the sample complexity for private
  learning and private data release}.
\newblock \bibinfo{journal}{\emph{Machine learning}} \bibinfo{volume}{94},
  \bibinfo{number}{3} (\bibinfo{year}{2014}), \bibinfo{pages}{401--437}.
\newblock


\bibitem[\protect\citeauthoryear{Bun and Steinke}{Bun and Steinke}{2016}]%
        {Bun2016zCDP}
\bibfield{author}{\bibinfo{person}{Mark Bun} {and} \bibinfo{person}{Thomas
  Steinke}.} \bibinfo{year}{2016}\natexlab{}.
\newblock \showarticletitle{Concentrated differential privacy: Simplifications,
  extensions, and lower bounds}. In \bibinfo{booktitle}{\emph{Theory of
  Cryptography Conference}}. Springer, \bibinfo{pages}{635--658}.
\newblock


\bibitem[\protect\citeauthoryear{Chang and Lin}{Chang and Lin}{2011}]%
        {Chang2011libsvm}
\bibfield{author}{\bibinfo{person}{Chih-Chung Chang} {and}
  \bibinfo{person}{Chih-Jen Lin}.} \bibinfo{year}{2011}\natexlab{}.
\newblock \showarticletitle{LIBSVM: a library for support vector machines}.
\newblock \bibinfo{journal}{\emph{ACM transactions on intelligent systems and
  technology (TIST)}} \bibinfo{volume}{2}, \bibinfo{number}{3}
  (\bibinfo{year}{2011}), \bibinfo{pages}{27}.
\newblock


\bibitem[\protect\citeauthoryear{Chaudhuri, Monteleoni, and Sarwate}{Chaudhuri
  et~al\mbox{.}}{2011}]%
        {Chaudhuri2011Objpert}
\bibfield{author}{\bibinfo{person}{Kamalika Chaudhuri}, \bibinfo{person}{Claire
  Monteleoni}, {and} \bibinfo{person}{Anand~D Sarwate}.}
  \bibinfo{year}{2011}\natexlab{}.
\newblock \showarticletitle{Differentially private empirical risk
  minimization}.
\newblock \bibinfo{journal}{\emph{Journal of Machine Learning Research}}
  \bibinfo{volume}{12}, \bibinfo{number}{Mar} (\bibinfo{year}{2011}),
  \bibinfo{pages}{1069--1109}.
\newblock


\bibitem[\protect\citeauthoryear{Dwork, Kenthapadi, McSherry, Mironov, and
  Naor}{Dwork et~al\mbox{.}}{2006a}]%
        {Dwork2006our}
\bibfield{author}{\bibinfo{person}{Cynthia Dwork}, \bibinfo{person}{Krishnaram
  Kenthapadi}, \bibinfo{person}{Frank McSherry}, \bibinfo{person}{Ilya
  Mironov}, {and} \bibinfo{person}{Moni Naor}.}
  \bibinfo{year}{2006}\natexlab{a}.
\newblock \showarticletitle{Our data, ourselves: Privacy via distributed noise
  generation}. In \bibinfo{booktitle}{\emph{Annual International Conference on
  the Theory and Applications of Cryptographic Techniques}}. Springer,
  \bibinfo{pages}{486--503}.
\newblock


\bibitem[\protect\citeauthoryear{Dwork, McSherry, Nissim, and Smith}{Dwork
  et~al\mbox{.}}{2006b}]%
        {Dwork2006calibrating}
\bibfield{author}{\bibinfo{person}{Cynthia Dwork}, \bibinfo{person}{Frank
  McSherry}, \bibinfo{person}{Kobbi Nissim}, {and} \bibinfo{person}{Adam
  Smith}.} \bibinfo{year}{2006}\natexlab{b}.
\newblock \showarticletitle{Calibrating noise to sensitivity in private data
  analysis}. In \bibinfo{booktitle}{\emph{Theory of Cryptography Conference}}.
  Springer, \bibinfo{pages}{265--284}.
\newblock


\bibitem[\protect\citeauthoryear{Dwork, Roth, et~al\mbox{.}}{Dwork
  et~al\mbox{.}}{2014}]%
        {Dwork2014DPbook}
\bibfield{author}{\bibinfo{person}{Cynthia Dwork}, \bibinfo{person}{Aaron
  Roth}, {et~al\mbox{.}}} \bibinfo{year}{2014}\natexlab{}.
\newblock \showarticletitle{The algorithmic foundations of differential
  privacy}.
\newblock \bibinfo{journal}{\emph{Foundations and Trends{\textregistered} in
  Theoretical Computer Science}} \bibinfo{volume}{9}, \bibinfo{number}{3--4}
  (\bibinfo{year}{2014}), \bibinfo{pages}{211--407}.
\newblock


\bibitem[\protect\citeauthoryear{Dwork, Rothblum, and Vadhan}{Dwork
  et~al\mbox{.}}{2010}]%
        {Dwork2010boosting}
\bibfield{author}{\bibinfo{person}{C. Dwork}, \bibinfo{person}{G.~N. Rothblum},
  {and} \bibinfo{person}{S. Vadhan}.} \bibinfo{year}{2010}\natexlab{}.
\newblock \showarticletitle{Boosting and Differential Privacy}. In
  \bibinfo{booktitle}{\emph{2010 IEEE 51st Annual Symposium on Foundations of
  Computer Science}}. \bibinfo{pages}{51--60}.
\newblock


\bibitem[\protect\citeauthoryear{Jain, Kothari, and Thakurta}{Jain
  et~al\mbox{.}}{2012}]%
        {Jain2012online}
\bibfield{author}{\bibinfo{person}{Prateek Jain}, \bibinfo{person}{Pravesh
  Kothari}, {and} \bibinfo{person}{Abhradeep Thakurta}.}
  \bibinfo{year}{2012}\natexlab{}.
\newblock \showarticletitle{Differentially private online learning}. In
  \bibinfo{booktitle}{\emph{Conference on Learning Theory}}.
  \bibinfo{pages}{24--1}.
\newblock


\bibitem[\protect\citeauthoryear{Kifer, Smith, and Thakurta}{Kifer
  et~al\mbox{.}}{2012}]%
        {Kifer2012erm}
\bibfield{author}{\bibinfo{person}{Daniel Kifer}, \bibinfo{person}{Adam Smith},
  {and} \bibinfo{person}{Abhradeep Thakurta}.} \bibinfo{year}{2012}\natexlab{}.
\newblock \showarticletitle{Private convex empirical risk minimization and
  high-dimensional regression}. In \bibinfo{booktitle}{\emph{Conference on
  Learning Theory}}. \bibinfo{pages}{25--1}.
\newblock


\bibitem[\protect\citeauthoryear{Lichman}{Lichman}{2013}]%
        {Lichman2013UCI}
\bibfield{author}{\bibinfo{person}{M. Lichman}.}
  \bibinfo{year}{2013}\natexlab{}.
\newblock \bibinfo{title}{{UCI} Machine Learning Repository}.
\newblock   (\bibinfo{year}{2013}).
\newblock
\urldef\tempurl%
\url{http://archive.ics.uci.edu/ml}
\showURL{%
\tempurl}


\bibitem[\protect\citeauthoryear{Nocedal}{Nocedal}{1980}]%
        {Nocedal1980updating}
\bibfield{author}{\bibinfo{person}{Jorge Nocedal}.}
  \bibinfo{year}{1980}\natexlab{}.
\newblock \showarticletitle{Updating quasi-Newton matrices with limited
  storage}.
\newblock \bibinfo{journal}{\emph{Mathematics of computation}}
  \bibinfo{volume}{35}, \bibinfo{number}{151} (\bibinfo{year}{1980}),
  \bibinfo{pages}{773--782}.
\newblock


\bibitem[\protect\citeauthoryear{Nocedal and Wright}{Nocedal and
  Wright}{2006}]%
        {Nocedal2006NO}
\bibfield{author}{\bibinfo{person}{J. Nocedal} {and} \bibinfo{person}{S.~J.
  Wright}.} \bibinfo{year}{2006}\natexlab{}.
\newblock \bibinfo{booktitle}{\emph{Numerical Optimization}
  (\bibinfo{edition}{2nd} ed.)}.
\newblock \bibinfo{publisher}{Springer}, \bibinfo{address}{New York}.
\newblock


\bibitem[\protect\citeauthoryear{Robbins and Monro}{Robbins and Monro}{1951}]%
        {Robbins1951stochastic}
\bibfield{author}{\bibinfo{person}{Herbert Robbins} {and}
  \bibinfo{person}{Sutton Monro}.} \bibinfo{year}{1951}\natexlab{}.
\newblock \showarticletitle{A stochastic approximation method}.
\newblock \bibinfo{journal}{\emph{The annals of mathematical statistics}}
  (\bibinfo{year}{1951}), \bibinfo{pages}{400--407}.
\newblock


\bibitem[\protect\citeauthoryear{Rubinstein, Bartlett, Huang, and
  Taft}{Rubinstein et~al\mbox{.}}{2012}]%
        {Rubinstein2012learning}
\bibfield{author}{\bibinfo{person}{Benjamin~IP Rubinstein},
  \bibinfo{person}{Peter~L Bartlett}, \bibinfo{person}{Ling Huang}, {and}
  \bibinfo{person}{Nina Taft}.} \bibinfo{year}{2012}\natexlab{}.
\newblock \showarticletitle{Learning in a Large Function Space:
  Privacy-Preserving Mechanisms for SVM Learning}.
\newblock \bibinfo{journal}{\emph{Journal of Privacy and Confidentiality}}
  \bibinfo{volume}{4}, \bibinfo{number}{1} (\bibinfo{year}{2012}),
  \bibinfo{pages}{4}.
\newblock


\bibitem[\protect\citeauthoryear{Ruggles, Genadek, Goeken, Grover, and
  Sobek}{Ruggles et~al\mbox{.}}{[n. d.]}]%
        {IPUMS}
\bibfield{author}{\bibinfo{person}{Steven Ruggles}, \bibinfo{person}{Katie
  Genadek}, \bibinfo{person}{Ronald Goeken}, \bibinfo{person}{Josiah Grover},
  {and} \bibinfo{person}{Matthew Sobek}.} \bibinfo{year}{[n. d.]}\natexlab{}.
\newblock \bibinfo{title}{Integrated Public Use Microdata Series, Minnesota
  Population Center}.
\newblock \bibinfo{howpublished}{\url{http://international.ipums.org}}.
  (\bibinfo{year}{[n. d.]}).
\newblock


\bibitem[\protect\citeauthoryear{Song, Chaudhuri, and Sarwate}{Song
  et~al\mbox{.}}{2013}]%
        {song13}
\bibfield{author}{\bibinfo{person}{S. Song}, \bibinfo{person}{K. Chaudhuri},
  {and} \bibinfo{person}{A.~D. Sarwate}.} \bibinfo{year}{2013}\natexlab{}.
\newblock \showarticletitle{Stochastic gradient descent with differentially
  private updates}. In \bibinfo{booktitle}{\emph{GlobalSIP}}.
\newblock


\bibitem[\protect\citeauthoryear{Talwar, Thakurta, and Zhang}{Talwar
  et~al\mbox{.}}{2015}]%
        {Talwar2015nearly}
\bibfield{author}{\bibinfo{person}{Kunal Talwar},
  \bibinfo{person}{Abhradeep~Guha Thakurta}, {and} \bibinfo{person}{Li Zhang}.}
  \bibinfo{year}{2015}\natexlab{}.
\newblock \showarticletitle{Nearly optimal private lasso}. In
  \bibinfo{booktitle}{\emph{Advances in Neural Information Processing
  Systems}}. \bibinfo{pages}{3025--3033}.
\newblock


\bibitem[\protect\citeauthoryear{Wang, Ye, and Xu}{Wang et~al\mbox{.}}{2017}]%
        {Wang2017dpsvrg}
\bibfield{author}{\bibinfo{person}{Di Wang}, \bibinfo{person}{Minwei Ye}, {and}
  \bibinfo{person}{Jinhui Xu}.} \bibinfo{year}{2017}\natexlab{}.
\newblock \showarticletitle{Differentially Private Empirical Risk Minimization
  Revisited: Faster and More General}. In \bibinfo{booktitle}{\emph{Advances in
  Neural Information Processing Systems 30}}. \bibinfo{publisher}{Curran
  Associates, Inc.}, \bibinfo{pages}{2719--2728}.
\newblock


\bibitem[\protect\citeauthoryear{Wang, Fienberg, and Smola}{Wang
  et~al\mbox{.}}{2015}]%
        {Wang2015free}
\bibfield{author}{\bibinfo{person}{Yu-Xiang Wang}, \bibinfo{person}{Stephen
  Fienberg}, {and} \bibinfo{person}{Alex Smola}.}
  \bibinfo{year}{2015}\natexlab{}.
\newblock \showarticletitle{Privacy for free: Posterior sampling and stochastic
  gradient monte carlo}. In \bibinfo{booktitle}{\emph{International Conference
  on Machine Learning}}. \bibinfo{pages}{2493--2502}.
\newblock


\bibitem[\protect\citeauthoryear{Williams and McSherry}{Williams and
  McSherry}{2010}]%
        {Williams2010probabilistic}
\bibfield{author}{\bibinfo{person}{Oliver Williams} {and}
  \bibinfo{person}{Frank McSherry}.} \bibinfo{year}{2010}\natexlab{}.
\newblock \showarticletitle{Probabilistic inference and differential privacy}.
  In \bibinfo{booktitle}{\emph{Proceedings of the 23rd International Conference
  on Neural Information Processing Systems-Volume 2}}. Curran Associates Inc.,
  \bibinfo{pages}{2451--2459}.
\newblock


\bibitem[\protect\citeauthoryear{Zhang, Xiao, Yang, Zhang, and Winslett}{Zhang
  et~al\mbox{.}}{2013}]%
        {Zhang2013privgene}
\bibfield{author}{\bibinfo{person}{Jun Zhang}, \bibinfo{person}{Xiaokui Xiao},
  \bibinfo{person}{Yin Yang}, \bibinfo{person}{Zhenjie Zhang}, {and}
  \bibinfo{person}{Marianne Winslett}.} \bibinfo{year}{2013}\natexlab{}.
\newblock \showarticletitle{PrivGene: Differentially Private Model Fitting
  Using Genetic Algorithms}. In \bibinfo{booktitle}{\emph{Proceedings of the
  2013 ACM SIGMOD International Conference on Management of Data}}
  \emph{(\bibinfo{series}{SIGMOD '13})}. \bibinfo{publisher}{ACM},
  \bibinfo{address}{New York, NY, USA}, \bibinfo{pages}{665--676}.
\newblock


\bibitem[\protect\citeauthoryear{Zhang, Zheng, Mou, and Wang}{Zhang
  et~al\mbox{.}}{2017}]%
        {Zhang2017RR}
\bibfield{author}{\bibinfo{person}{Jiaqi Zhang}, \bibinfo{person}{Kai Zheng},
  \bibinfo{person}{Wenlong Mou}, {and} \bibinfo{person}{Liwei Wang}.}
  \bibinfo{year}{2017}\natexlab{}.
\newblock \showarticletitle{Efficient private ERM for smooth objectives}. In
  \bibinfo{booktitle}{\emph{Proceedings of the 26th International Joint
  Conference on Artificial Intelligence}}. AAAI Press,
  \bibinfo{pages}{3922--3928}.
\newblock


\end{thebibliography}

\end{document}